%% file: example_paper.tex
\documentclass[nohyperref]{article}
\input{math_commands.tex}

\usepackage{microtype}
\usepackage{graphicx}
\usepackage{booktabs} %

\usepackage{hyperref}

\usepackage[accepted]{icml2022}

\usepackage{amsmath}
\usepackage{amssymb}
\usepackage{mathtools}
\usepackage{amsthm}

\usepackage[capitalize,noabbrev]{cleveref}

\usepackage[textsize=tiny]{todonotes}

\usepackage{subcaption}
\usepackage{amsmath}
\usepackage{amsthm}
\usepackage{amssymb}
\usepackage{hhline}
\usepackage{multirow}
\usepackage{wrapfig}
\usepackage{caption}
\usepackage{diagbox}
\usepackage{slashbox}
\usepackage{bbm}
\usepackage{tabularx}

\usepackage{thmtools}
\usepackage{thm-restate}

\usepackage{xpatch}
\usepackage{esint}

\makeatletter
\@for\next:={int,iint,iiint,iiiint,dotsint,oint,oiint,sqint,sqiint,
  ointctrclockwise,ointclockwise,varointclockwise,varointctrclockwise,
  fint,varoiint,landupint,landdownint}\do{%
    \expandafter\edef\csname\next\endcsname{%
      \noexpand\DOTSI
      \expandafter\noexpand\csname\next op\endcsname
      \noexpand\ilimits@
    }%
  }
\makeatother

\makeatletter
\xpatchcmd{\thmt@restatable}%
{\csname #2\@xa\endcsname\ifx\@nx#1\@nx\else[{#1}]\fi}%
{\ifthmt@thisistheone\csname #2\@xa\endcsname\ifx\@nx#1\@nx\else[{#1}]\fi\else\csname #2\@xa\endcsname\fi}%
{}{} %
\makeatother

\makeatletter
\newcommand{\pushright}[1]{\ifmeasuring@#1\else\omit\hfill$\displaystyle#1$\fi\ignorespaces}
\newcommand{\pushleft}[1]{\ifmeasuring@#1\else\omit$\displaystyle#1$\hfill\fi\ignorespaces}
\makeatother

\newcommand*{\Scale}[2][4]{\scalebox{#1}{$#2$}}%

\newcommand{\vect}[1]{\boldsymbol{#1}}
\newcommand{\RN}[1]{%
    \textup{\lowercase\expandafter{\it \romannumeral#1}}%
}

\usepackage{array}

\icmltitlerunning{NeuralEF: Deconstructing Kernels by Deep Neural Networks}

\begin{document}

\twocolumn[
\icmltitle{NeuralEF: Deconstructing Kernels by Deep Neural Networks}

\begin{icmlauthorlist}
\icmlauthor{Zhijie Deng}{thu,sjtu}
\icmlauthor{Jiaxin Shi}{msr}
\icmlauthor{Jun Zhu}{thu}
\end{icmlauthorlist}

\icmlaffiliation{thu}{Dept. of Comp. Sci. \& Tech., BNRist Center, Tsinghua-Bosch Joint Center for ML, Tsinghua University; Peng Cheng Laboratory}
\icmlaffiliation{sjtu}{Qingyuan Research, Shanghai Jiao Tong University}
\icmlaffiliation{msr}{Microsoft Research New England}

\icmlcorrespondingauthor{Jun Zhu}{dcszj@tsinghua.edu.cn}

\icmlkeywords{kernel methods, neural networks, eigendecomposition, kernel approximation}

\vskip 0.3in
]

\printAffiliationsAndNotice{}  %

\begin{abstract}
Learning the principal eigenfunctions of an integral operator defined by a kernel and a data distribution is at the core of many machine learning problems. Traditional nonparametric solutions based on the Nystr{\"o}m formula suffer from scalability issues. Recent work has resorted to a parametric approach, i.e., training neural networks to approximate the eigenfunctions. However, the existing method relies on an expensive orthogonalization step and is difficult to implement. We show that these problems can be fixed by using a new series of objective functions that generalizes the EigenGame~\citep{gemp2020eigengame} to function space. We test our method on a variety of supervised and unsupervised learning problems and show it provides accurate approximations to the eigenfunctions of polynomial, radial basis, neural network Gaussian process, and neural tangent kernels. Finally, we demonstrate our method can scale up linearised Laplace approximation of deep neural networks to modern image classification datasets through approximating the Gauss-Newton matrix. Code is available at \url{https://github.com/thudzj/neuraleigenfunction}. 
\end{abstract}

\section{Introduction}

Kernel methods~\cite{muller2001introduction,scholkopf2002learning} are among the most important and fundamental tools in machine learning (ML) for processing nonlinear data.
Like most nonparametric methods, kernel methods have limited applicability in large-scale scenarios, and kernel approximation methods like \emph{Random Fourier features} (RFFs)~\cite{rahimi2007random} and \emph{Nystr\"{o}m method}~\cite{nystrom1930praktische} are frequently introduced as a remedy. 

Even so, kernel methods still lag behind in handling complex high-dimensional data as classic local kernels suffer from the curse of dimensionality~\citep{bengio2005curse}. 
New developments such as neural network Gaussian process (NN-GP) kernels~\cite{lee2017deep} and neural tangent kernels (NTKs)~\cite{jacot2018neural} mitigate this issue by leveraging the inductive bias of neural networks (NNs). %
Yet, these modern kernels are expensive to evaluate and have poor compatibility with standard kernel approximation methods---an accurate approximation with RFFs requires many NN forward passes to construct an adequate number of features, while Nystr\"{o}m method involves costly kernel evaluation in the test phase.

We present a new kernel approximation method for addressing these challenges. 
It relies on deconstructing a kernel as a spectral series and approximating the eigenfunctions with neural networks. 
Because such a method can identify the principal components, we need much less NN forward passes than RFFs to accurately approximate NN-GP kernels and NTKs. 
Moreover, it gives rise to an unsupervised representation learning paradigm, where %
the pairwise similarity captured by kernels is embedded into NNs.  %
The learned neural eigenfunctions can serve as plug-and-play feature extractors and be finetuned for downstream applications.

\begin{figure}[t] %
\vskip 0.2in
\centering
\includegraphics[width=\linewidth]{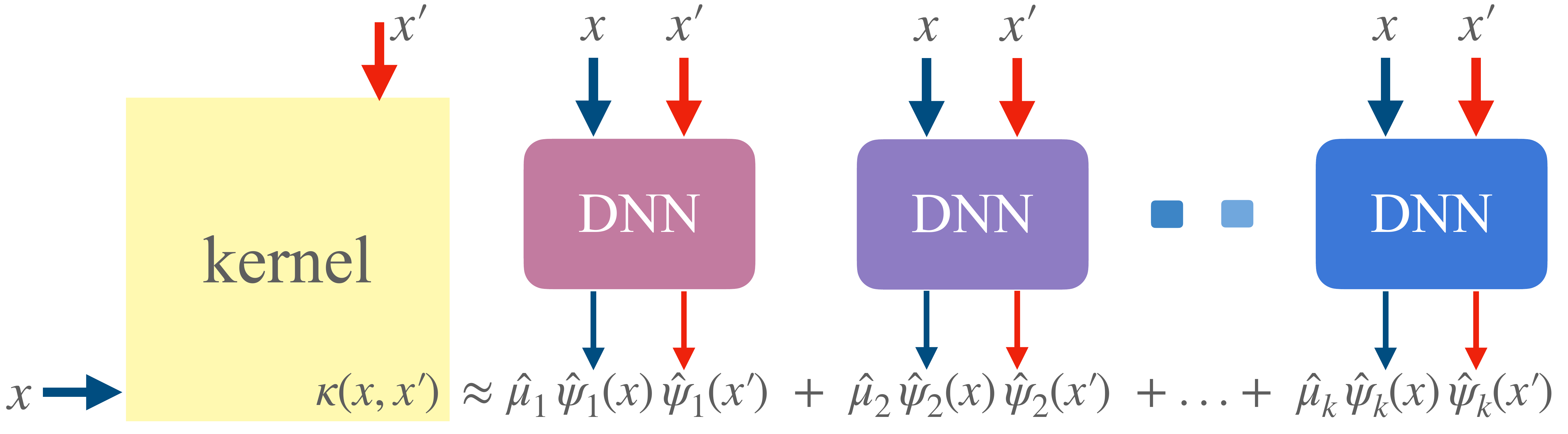}
\caption{\small %
NeuralEF embeds the structures inside a kernel into an ensemble of NNs under the principle of eigendecomposition.
NeuralEF defines a new axis system for data, and is capable of recovering the kernel according to Mercer's theorem.}
\label{fig:fw}
\end{figure}

\begin{figure*}[t]
\centering
    \centering
    \includegraphics[width=\linewidth]{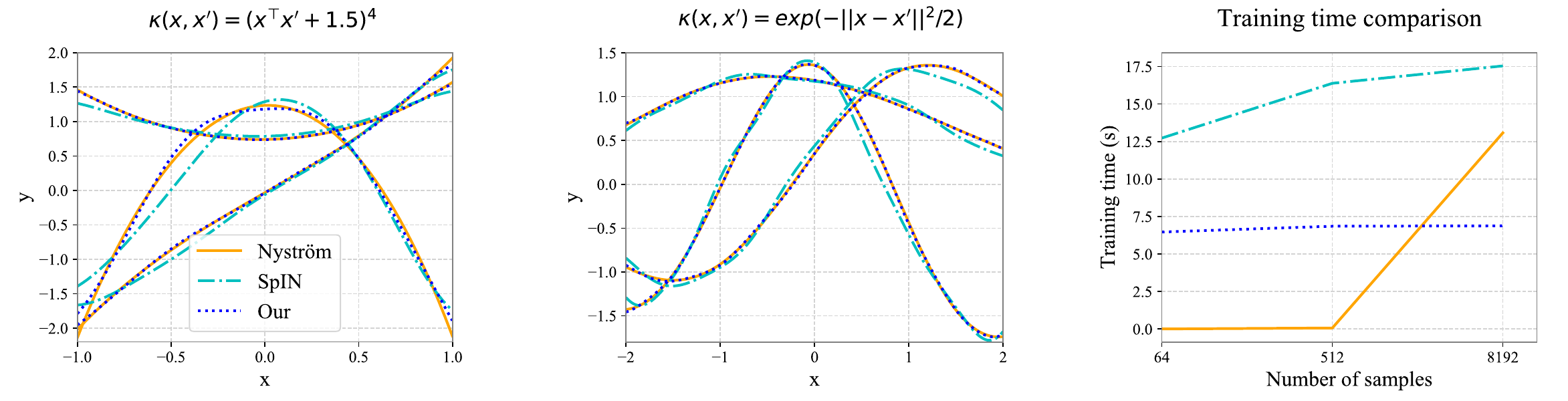}
        \vspace{-3.ex}
\caption{\footnotesize Estimate the top-$10$ eigenfunctions $\{\hat{\psi}_j\}_{j=1}^{10}$ (we plot only the top-$3$ for visualization) of the polynomial kernel $\kappa(x,x')=(x^\top x' + 1.5)^4$ (Left) and the RBF kernel $\kappa(x,x')=\exp(-\Vert x - x'\Vert^2 / 2)$ (Middle) with the Nystr\"{o}m method (the ground truth), SpIN, and NeuralEF (our) given $64$ data points. 
The results obtained with larger sample size are deferred to \cref{app:classic}. 
We also plot how the total training time varies w.r.t. sample size for the polynomial kernel (Right). 
The time complexity of Nystr\"om method grows cubically with the sample size due to eigendecomposition. 
Both NeuralEF and SpIN %
consume nearly constant training time for different sample sizes, but
SpIN produces less accurate approximations and is less efficient. 
} 
\label{fig:illu}
\end{figure*}

Training neural networks to approximate eigenfunctions is a difficult task due to the orthonormality constraints. 
Spectral Inference Networks (SpIN)~\cite{pfau2018spectral} made the first attempt by formulating this as an optimization problem. 
Nevertheless, SpIN's objective function only allows recovering the subspace spanned by the top eigenfunctions\footnote{By saying top-$k$ eigenpairs or eigenfunctions, we mean those associated with the top-$k$ eigenvalues. Besides, we assume they are ranked in the descending order of the corresponding eigenvalues without loss of generality.} instead of the eigenfunctions themselves. 
To recover individual eigenfunctions, they %
resort to an expensive orthogonalization step that requires Cholesky decomposition and tracking the NN Jacobian matrix during training. 
As a result, SpIN suffer from inefficiency issues and non-trivial implementation challenges.

We resolve the problems of SpIN with a %
new formulation of kernel eigendecomposition 
 as \emph{simultaneously} solving a series of \emph{asymmetric} optimization problems. 
Our work is inspired from \citet{gemp2020eigengame} %
which uses a related objective to cast principled component analysis as a game, 
and %
can be viewed as an extension of it
to the function space.

With this, we can naturally incorporate $k$ NNs as function approximators to the top-$k$ eigenfunctions of the kernel, and train them under a more amenable objective than SpIN. 
We leverage stochastic optimization strategies to comfortably train these neural eigenfunctions in the big data regime, and develop a novel NN layer to make the neural eigenfunctions fulfill certain constraints. 
We dub our method as \emph{neural eigenfunctions}, or \emph{NeuralEF} for short.%

After training, NeuralEF approximates the kernel evaluation on novel data points with the outputs of the $k$ NNs (see \cref{fig:fw}). 
This brings great benefits when coping with modern kernels whose evaluation incurs heavy burden like NN-GP kernels and NTKs. 
Besides, NeuralEF forms a set of new orthogonal bases for the data and is compatible with various pattern recognition algorithms.

We empirically evaluate NeuralEF in a variety of scenarios. 
We first evidence that NeuralEF can perform as well as the Nystr\"{o}m method and beat SpIN when handling classic kernels, yet with more sustainable resource consumption (see \cref{fig:illu}). 
We then show that NeuralEF can learn the informative structures in NN-GP kernels and NTKs to achieve effective classification and clustering. 
To demonstrate the scalability of NeuralEF, we further conduct a large-scale study to
scale up linearised Laplace approximation %
where the kernel matrix has size $500,000 \times 500,000$. %

\section{Background}

\subsection{Kernel Methods}
Kernel methods project %
data %
into a high-dimensional feature space $\mathcal{H}$ to enable linear manipulation of nonlinear data. 
The subtlety is that we can leverage the ``kernel trick'' to bypass the need
of specifying $\mathcal{H}$ explicitly
---given a positive definite kernel $\kappa: \mathcal{X} \times \mathcal{X} \rightarrow \mathbb{R}$, there exists a feature map $\varphi: \mathcal{X}\rightarrow\mathcal{H}$ such that $\kappa(\vx, \vx') = \langle\varphi(\vx), \varphi(\vx')\rangle_\mathcal{H}$.  %

There is a rich family of kernels.
Classic kernels include the linear kernel, the polynomial kernel, the radial basis function (RBF) kernel, etc. %
However, they may easily fail when processing real-world data like images and texts due to inadequate expressive power and the curse of dimensionality. 
Thereby, various modern kernels which encapsulate the inductive biases of NN architectures have been developed~\cite{wilson2016deep}. 
The NN-GP kernels~\cite{neal1996priors,lee2017deep,garriga2018deep,novak2018bayesian} and NTKs~\cite{jacot2018neural,arora2019exact,du2019graph} are two representative modern \emph{matrix-valued} kernels, as defined below:
\begin{align}
\small
\kappa_\text{NN-GP}(\vx, \vx') &= \mathbb{E}_{\vtheta \sim p(\vtheta)} g(\vx, \vtheta)g(\vx', \vtheta)^\top, \\
\kappa_\text{NTK}(\vx, \vx') &= \mathbb{E}_{\vtheta \sim p(\vtheta)} \partial_{\vtheta}g(\vx, \vtheta) \partial_{\vtheta}g(\vx', \vtheta)^\top,
\end{align}
where $g(\cdot, \vtheta): \mathcal{X} \rightarrow \mathbb{R}^{N_\text{out}}$ denotes a function represented by an NN with weights $\vtheta$, and $p(\vtheta)$ is a layerwise isotropic Gaussian. 
When the NN specifying $g$ has infinite width, $\kappa_\text{NN-GP}$ and $\kappa_\text{NTK}$ have analytical formulae and can be computed recursively 
due to the hierarchical nature of NNs. 
The NTKs are valuable for the theoretical analysis of neural networks, but are used much less frequently in practice than the empirical NTKs defined below:
\begin{align}
\small
\kappa_\text{NTK}(\vx, \vx') &= \partial_{\vtheta}g(\vx, \vtheta) \partial_{\vtheta}g(\vx', \vtheta)^\top.
\end{align}
If there is no misleading, we refer to the empirical NTKs as NTKs in the following.

Despite the non-parametric flexibility, %
kernel methods suffer from inefficiency issues -- the involved computations grow at least quadratically and usually cubically w.r.t. data size. 
Moreover, for NN-GP kernels and NTKs,
writing down their detailed mathematical formulae is non-trivial~\cite{arora2019exact} 
and evaluating them with recursion is both time and memory consuming. 

\subsection{Kernel Approximation}
\label{sec:kernelapprox}
An idea to scale up kernel methods is to approximate the kernel %
with the inner product of some explicit vector representations of the data, i.e., $\kappa(\vx,\vx') \approx \nu(\vx)^\top\nu(\vx')$, where $\nu: \mathcal{X}\rightarrow \mathbb{R}^k$ denotes a mapping function. 
\emph{Random Fourier features} (RFFs)~\cite{rahimi2007random,rahimi2008weighted,yu2016orthogonal,munkhoeva2018quadrature,francis2021major} and \emph{Nystr\"{o}m method}~\cite{nystrom1930praktische,williams2001using} are two popular approaches in this spirit. 
RFFs can easily handle shift-invariant kernels, but may face obstacles when being applied to the others. Besides, RFFs usually entail using a relatively large $k$.

Alternatively, Nystr\"{o}m method finds the eigenfunctions for kernel approximation according to Mercer's theorem:
\begin{equation}
\label{eq:mercer}
\small
\kappa(\vx, \vx') = \sum_{j\geq1}\mu_j \psi_j(\vx)\psi_j(\vx'),  %
\end{equation}
where $\psi_j \in L^2(\mathcal{X}, q)$\footnote{$L^2(\mathcal{X}, q)$ is the space containing all square-integrable functions w.r.t. $q$.} denote the eigenfunctions of the kernel $\kappa(\vx, \vx')$ w.r.t. the probability measure $q$, and $\mu_j \geq 0$ refer to the corresponding eigenvalues.

Typically, the eigenfunctions obey the following equation:
\begin{equation}
\label{eq:eigenfunc}
\small
\int \kappa(\vx, \vx')\psi_j(\vx')q(\vx')d\vx' = \mu_j\psi_j(\vx),\;\forall j \geq 1.
\end{equation}
And they are demanded to be orthonormal under $q$:
\begin{equation}
\small
\label{eq:eigenfunc-c}
\int \psi_i(\vx)\psi_j(\vx)q(\vx)d\vx =\mathbbm{1}[i=j],\;\forall i,j \geq 1.
\end{equation}
Given %
the access to a training set of i.i.d. samples $\mathbf{X}_\text{tr}=\{\vx_1,...,\vx_N\}$ from $q$, the Nystr\"{o}m method approximates the integration in \cref{eq:eigenfunc} by Monte Carlo (MC) estimation, which gives rise to 
\begin{equation}
\small
\label{eq:5}
\frac{1}{N}\sum_{n'=1}^N\kappa(\vx, \vx_{n'})\psi_j(\vx_{n'}) = \mu_j \psi_j(\vx), \forall j \geq 1.
\end{equation}
It is easy to see that we can get the approximate top-$k$ eigenpairs $\{(\hat{\mu}_j, \Scale[0.95]{[\hat{\psi}_j(\vx_1),...,\hat{\psi}_j(\vx_N)]^\top})\}_{j=1}^k$ of the kernel $\kappa$ by eigendecomposing the kernel matrix $\Scale[0.95]{\kappa(\mathbf{X}_\text{tr}, \mathbf{X}_\text{tr})\in\mathbb{R}^{N \times N}}$.
Nystr\"{o}m method plugs them back into \cref{eq:5} for out-of-sample extension:
\begin{equation}
\small
\label{eq:nystrom}
\hat{\psi}_j(\vx) = \frac{1}{N\hat{\mu}_j}\sum_{{n'}=1}^N\kappa(\vx, \vx_{n'})\hat{\psi}_j(\vx_{n'}), \;j \in [k].
\end{equation}
$[k]$ %
denotes set of integers from $1$ to $k$. Then, we have $\kappa(\vx, \vx') \approx \sum_{j=1}^k \hat{\mu}_j \hat{\psi}_j(\vx)\hat{\psi}_j(\vx').$ 

Unlike RFFs, Nystr{\"o}m method can be applied to approximate any positive-definite kernel. 
Yet, it is non-trivial to scale it up. 
On the one hand, the matrix eigendecomposition is costly for even medium sized training data. 
On the other hand, as shown in \cref{eq:nystrom}, evaluating $\hat{\psi}_j$ on a new datum entails evaluating $\kappa$ for $N$ times, which is unaffordable when coping with the modern kernels specified by deep architectures.

\subsection{Spectral Inference Networks (SpIN)}
Kernel approximation with NNs has the potential to ameliorate these pathologies due to their universal approximation capability and parametric nature. 
SpIN~\cite{pfau2018spectral} is a pioneering work in this line. 
In a similar spirit to the Nystr\"{o}m method, SpIN recovers the top eigenfunctions for kernel approximation, yet with NNs.
Concretely, SpIN introduces a vector-valued NN function $\Psi(\cdot, \vw):\mathcal{X}\rightarrow \mathbb{R}^k$ and train it under the following eigendecomposition principle:
\begin{align}
\small
\max_{\vw} \mathrm{Tr}\Big(\iint &\Psi(\vx, \vw)\Psi(\vx', \vw)^\top \kappa(\vx, \vx') q(\vx)q(\vx')d\vx d\vx' \Big) \notag \\
\text{s.t.:}\, &\int \Psi(\vx, \vw)\Psi(\vx, \vw)^\top  q(\vx)d\vx = \mathbf{I}_k,  \label{eq:spin}
\end{align}
where $\mathrm{Tr}$ computes the trace of a matrix and $\mathbf{I}_{k}$ denotes the identity matrix of size $k \times k$; $q$ is the empirical distribution of data.

However, the above objective makes $\Psi$ recover the subspace spanned by the top-$k$ eigenfunctions rather than the top-$k$ eigenfunctions themselves~\cite{pfau2018spectral}. 
Such a difference is subtle yet significant. 
In the extreme case of finding the top-$N$ eigenfunctions given a training set $\mathbf{X}_\text{tr}$ of size $N$, \cref{eq:spin} becomes: 
\begin{equation}
\small
\label{eq:spin1}
\begin{aligned}
&\max_\vw \mathrm{Tr}\left(\Psi(\mathbf{X}_\text{tr}, \vw)^\top \kappa(\mathbf{X}_\text{tr}, \mathbf{X}_\text{tr})\Psi(\mathbf{X}_\text{tr}, \vw)\right) \\
\text{s.t.:}\, &\Psi(\mathbf{X}_\text{tr}, \vw) \in \mathbb{R}^{N \times N}, \, \Psi(\mathbf{X}_\text{tr}, \vw)^\top\Psi(\mathbf{X}_\text{tr}, \vw) = \mathbf{I}_N,
\end{aligned}
\end{equation}
which equals to $\Scale[0.9]{\max_\vw \mathrm{Tr}\left(\Psi(\mathbf{X}_\text{tr}, \vw)\Psi(\mathbf{X}_\text{tr}, \vw)^\top \kappa(\mathbf{X}_\text{tr}, \mathbf{X}_\text{tr})\right)=}$
$\Scale[0.9]{ \mathrm{Tr}\left(\kappa(\mathbf{X}_\text{tr}, \mathbf{X}_\text{tr})\right)}$. 
As the last term shows, the objective is independent of $\vw$, thus it cannot be used to learn all $N$ eigenfunctions. 

To fix this issue, SpIN relies on a gradient masking trick to ensure that $\Psi$ converges to ordered eigenfunctions. 
This solution is expensive as it involves a Cholesky decomposition per training iteration. 
Besides, to debias the stochastic optimization, SpIN involves tracking the exponential moving average (EMA) of the Jacobian $\partial_{\vw}\Sigma_\vw$. 
As a result, SpIN suffers from inefficiency issues (see \cref{fig:illu}) and non-trivial implementation challenges.
The explicit manipulation of Jacobian also precludes SpIN from utilizing very deep architectures. 

\section{Methodology}

\subsection{Eigendecomposition as Maximization Problems}
NeuralEF builds upon a new representation of the eigenpairs of the kernel, which leads to a more amenable learning principle for kernel approximation than those of the Nystr\"{o}m method and SpIN:

\begin{restatable}[Proof in \cref{app:proof}]{thm}{Thm}
\label{theorem:0}
The eigenpairs of the kernel $\kappa(\vx, \vx')$ can be recovered by simultaneously solving the following asymmetric maximization problems:
\vspace{-0.05cm}
\begin{equation}
\vspace{-0.05cm}
\label{eq:obj}
\small
\max_{\hat{\psi}_j} R_{jj} \;\;\textrm{s.t.:}\, C_j = 1, R_{ij} = 0, \forall j \geq 1,~ i\in[j-1],
\end{equation}
where $\hat{\psi}_j \in L^2(\mathcal{X}, q)$ represent the introduced approximate eigenfunctions, and
\begin{align}
\small
    R_{ij} &:=  \iint \hat{\psi}_i(\vx)\kappa(\vx, \vx')\hat{\psi}_j(\vx')q(\vx')q(\vx)d\vx'd\vx, \\
    C_j &:= \int \hat{\psi}_j(\vx)\hat{\psi}_j(\vx)q(\vx)d\vx.
\end{align}
$(R_{jj}, \hat{\psi}_j)$ will converge to the eigenpair associated with $j$-th largest eigenvalue of $\kappa$. 
\end{restatable}

Intuitively, $C_j=1$ guarantees the normalization of $\hat{\psi}_j$.
$R_{ij} = 0$ ($\forall i < j$) ensures that $\hat{\psi}_j$ and $\hat{\psi}_i$ are orthogonal, and induces a hierarchy among the problems -- the solutions to the latter problems are restricted in the \emph{orthogonal complement} of the subspace spanned by the solutions to the former ones. 
This way, we bypass the reliance on the troublesome Cholesky decomposition, gradient masking, and manipulation of Jacobian as required by SpIN. %

Conceptually, $R_{jj}$ can be thought of as an extension of the Rayleigh quotient. 
It is then interesting to see that our optimization objective makes an analogy with that of the 
EigenGame~\cite{gemp2020eigengame} for finite-dimensional symmetric matrices.
We clarify that our method is a generalization of EigenGame -- by setting $\kappa$ as the linear kernel $\kappa(\vx,\vx')=\vx^\top \vx'$, our objective forms a function-space generalization of that in EigenGame. %

In the seek of tractability, we slack the constraints on orthogonality in \cref{eq:obj} as penalties and solve:
\begin{equation}
\small
\label{eq:obj-k}
\max_{\hat{\psi}_j} R_{jj} - \sum_{i =1}^{j-1}\frac{R_{ij}^2}{R_{ii}} \;\; \text{s.t.:}\, C_j = 1, \forall j\in[k],
\end{equation}
where we set the penalty coefficients as $\Scale[0.95]{\frac{1}{R_{11}}, ..., \frac{1}{R_{(j-1)(j-1)}}}$ to make the two kinds of forces in the objective on a similar scale, as suggested by EigenGame~\cite{gemp2020eigengame}. 
We consider only the top-$k$ eigenpairs to strike a balance between efficiency and effectiveness.
$k\in \mathbb{N}^{+}$ is a tunable hyper-parameter. 
We will exposit how to handle the normalization constraints later.

Ideally, by \cref{theorem:0}, $\hat{\psi}_j$ would  converge to the ground truth eigenfunction $\psi_j$, but in practice the exact convergence may be unattainable. 
Given this situation, we provide an analysis in \cref{app:analysis} to demonstrate that small errors in the solutions to the former problems would not significantly bias the solutions to the latter problems.

\subsection{Neural Networks as Eigenfunctions}
\label{sec:method}
We opt to optimize over the parametric function class defined by NNs to scale up our method to large data. 
Concretely, we incorporate $k$ NNs with the same architecture but dedicated weights into the optimization in \cref{eq:obj-k}, i.e., $\Scale[0.95]{\hat{\psi}_j(\cdot)=\hat{\psi}(\cdot, \vw_j): \mathcal{X} \rightarrow \mathbb{R}, j\in[k]}$.

\textbf{Mini-batch training}  
We learn the parameters via mini-batch training. 
Given $\Scale[0.95]{\mathbf{X}=\{\vx_b\}_{b=1}^B \subseteq \mathbf{X}_\text{tr}}$ at per iteration,
we approximate $R_{ij}$ by MC integration:
\begin{equation}
\small
\begin{aligned}
\tilde{R}_{ij} &= \sum_{b=1}^B \sum_{b'=1}^B \frac{1}{B^2}\hat{\psi}_i(\vx_{b})\kappa(\vx_b, \vx_{b'})\hat{\psi}_j(\vx_{b'})\\
&= \frac{1}{B^2} {\vect{\hat{\psi}_i}^{\mathbf{X}}{}}^{\top} \vect{\kappa}^{\mathbf{X},\mathbf{X}} \vect{\hat{\psi}_j}^{\mathbf{X}},\\
\end{aligned}
\end{equation}
where $\vect{\hat{\psi}_j}^{\mathbf{X}} = [\hat{\psi}(\vx_1, \vw_j), ..., \hat{\psi}(\vx_B, \vw_j)]^\top \in \mathbb{R}^{B}$ refer to the concatenation of scalar outputs and $\vect{\kappa}^{\mathbf{X},\mathbf{X}} = \kappa(\mathbf{X}, \mathbf{X}) \in \mathbb{R}^{B\times B}$ is the kernel matrix associated with the mini-batch of data. 
We track $\tilde{R}_{jj}$ via EMA to get an estimate of the $j$-th largest eigenvalue ${\mu}_j$. 

\textbf{L2 Batch normalization (L2BN)} 
Then we settle the normalization constraints. %
When doing mini-batch training, we have $\Scale[0.95]{C_j \approx \frac{1}{B} \sum_{b=1}^B \hat{\psi}(\vx_b, \vw_j)\hat{\psi}(\vx_b, \vw_j)}$.
To make $\Scale[0.95]{C_j}$ equal to 1, we should guarantee the outcomes of each neural eigenfunction to have an L2 norm equal to $\Scale[0.88]{\sqrt{B}}$. 
To realise this, we append an L2 batch normalization (L2BN) layer at the end of every involved NN, whose form is:
\begin{equation}
\label{eq:l2bn}
\small
h_b^\text{out} = \frac{h_b^\text{in}}{\sigma}, \; \text{with}\; \sigma = \sqrt{\frac{1}{B}\sum_{b=1}^B {h_b^\text{in}}^{2}}, \;b \in[B].
\end{equation}

$\Scale[0.95]{\{h_b^\text{in}|h_b^\text{in}\in\mathbb{R}\}_{b=1}^B}$ and $\Scale[0.95]{\{h_b^\text{out}|h_b^\text{out}\in\mathbb{R}\}_{b=1}^B}$ denote the batched input and output. 
We apply the transformation in \cref{eq:l2bn} during training, while recording the EMA of $\sigma$ for testing. 
Thereby, we sidestep the barrier of optimizing under explicit normalization constraints. 
Arguably, L2BN also ensures the neural eigenfunctions are square-integrable w.r.t. $q$.

\textbf{The training loss and its gradients}
Given these setups, we arrive at the following training loss:
\begin{equation}
\small
\begin{aligned}
\label{eq:loss}
\min_{\vw_1, ..., \vw_k}\ell = &- \frac{1}{B^2} \sum_{j=1}^k \Bigg( {\vect{\hat{\psi}_j}^{\mathbf{X}}{}}^{\top} \vect{\kappa}^{\mathbf{X},\mathbf{X}} \vect{\hat{\psi}_j}^{\mathbf{X}} \\
&- \sum_{i=1}^{j-1}\frac{\left(\mathrm{sg}({\vect{\hat{\psi}_i}^{\mathbf{X}}{}}^{\top}) \vect{\kappa}^{\mathbf{X},\mathbf{X}} \vect{\hat{\psi}_j}^{\mathbf{X}}\right)^2}{\mathrm{sg}\left({\vect{\hat{\psi}_i}^{\mathbf{X}}{}}^{\top} \vect{\kappa}^{\mathbf{X},\mathbf{X}} \vect{\hat{\psi}_i}^{\mathbf{X}}\right)}\Bigg),
\end{aligned}
\end{equation}
where $\mathrm{sg}$ denotes the \texttt{stop\_gradient} operation.
The gradients of $\ell$ w.r.t. $\vw_j$ are:
\begin{equation}
\label{eq:grad}
\small
\Scale[0.95]{\nabla_{\vw_j} \ell = -\frac{2}{B^2} \vect{\kappa}^{\mathbf{X},\mathbf{X}}\left(\vect{\hat{\psi}_j}^{\mathbf{X}} - \sum_{i =1}^{j-1} \frac{{\vect{\hat{\psi}_i}^{\mathbf{X}}{}}^{\top} \vect{\kappa}^{\mathbf{X},\mathbf{X}} \vect{\hat{\psi}_j}^{\mathbf{X}}}{{\vect{\hat{\psi}_i}^{\mathbf{X}}{}}^{\top} \vect{\kappa}^{\mathbf{X},\mathbf{X}} \vect{\hat{\psi}_i}^{\mathbf{X}}}\vect{\hat{\psi}_i}^{\mathbf{X}}\right) \cdot \partial_{\vw_j} \vect{\hat{\psi}_j}^{\mathbf{X}}}.
\end{equation}
In practice, we analytically compute the first part of the RHS of \cref{eq:grad} and then perform vector-Jacobian product via reverse-mode autodiff.

\textbf{Extension to matrix-valued kernels} So far, we have figured out the manipulation of scalar kernels, and
it is natural to generalize NeuralEF to handle \emph{matrix-valued} kernels, whose outputs are $N_\text{out} \times N_\text{out}$ matrices.
In this regime, a direct solution is to set $\hat{\psi}_j$ as NNs with $N_\text{out}$ output neurons and modify the L2BN layer accordingly.
Then, we have $\vect{\kappa}^{\mathbf{X},\mathbf{X}} \in \mathbb{R}^{BN_\text{out} \times BN_\text{out}}$ and $\vect{\hat{\psi}_j}^{\mathbf{X}}=[\hat{\psi}(\vx_1, \vw_j)^\top, ..., \hat{\psi}(\vx_B, \vw_j)^\top]^\top \in \mathbb{R}^{BN_\text{out}}$. 
The loss $\ell$ is still leveraged to guide training.

\textbf{The algorithm} \cref{algo:1} shows the training procedure.\footnote{It is not indispensable to precompute the big $\vect{\kappa}^{\mathbf{X}_\text{tr},\mathbf{X}_\text{tr}}$ -- we can instead compute the small $\vect{\kappa}^{\mathbf{X},\mathbf{X}}$ at per training iteration.}

\begin{algorithm}[t] 
\caption{\small Find the top-$k$ eigenpairs of a kernel by NeuralEF} 
\label{algo:1}
\small
\begin{algorithmic}[1]
\STATE {\bfseries Input:} Training data $\mathbf{X}_\text{tr}$, kernel $\kappa$, batch size $B$.
\STATE{Initialize NNs $\Scale[0.95]{\hat{\psi}_j(\cdot)=\hat{\psi}(\cdot, \vw_j)}$ and scalars $\Scale[0.95]{\hat{\mu}_j, j\in [k]}$;}
\STATE{Compute the kernel matrix $\Scale[0.95]{\vect{\kappa}^{\mathbf{X}_\text{tr},\mathbf{X}_\text{tr}}=\kappa(\mathbf{X}_\text{tr}, \mathbf{X}_\text{tr})}$;}
\FOR{\emph{iteration}}
    \STATE {Draw a mini-batch $\Scale[0.95]{\mathbf{X} \subseteq \mathbf{X}_\text{tr}}$; retrieve $\Scale[0.95]{\vect{\kappa}^{\mathbf{X},\mathbf{X}}}$ from $\Scale[0.95]{\vect{\kappa}^{\mathbf{X}_\text{tr},\mathbf{X}_\text{tr}}}$;}
    \STATE {Do forward propagation $\Scale[0.95]{\vect{\hat{\psi}_j}^{\mathbf{X}} = \hat{\psi}(\mathbf{X}, \vw_j), j\in[k]}$;}
    \STATE{$\Scale[0.95]{\hat{\mu}_j \leftarrow \mathrm{EMA}(\hat{\mu}_j, \Scale[0.85]{\frac{1}{B^2} {\vect{\hat{\psi}_j}^{\mathbf{X}}{}}^{\top} \vect{\kappa}^{\mathbf{X},\mathbf{X}} \vect{\hat{\psi}_j}^{\mathbf{X}}}), j\in[k]}$;}
    \STATE {Compute $\Scale[0.95]{\nabla_{\vw_j}\ell, j\in[k]}$ by \cref{eq:grad} and do SGD;}
\ENDFOR
\end{algorithmic}
\end{algorithm}

\begin{figure*}[t]
\centering
\begin{subfigure}[b]{0.49\linewidth}
\centering
\includegraphics[width=\linewidth]{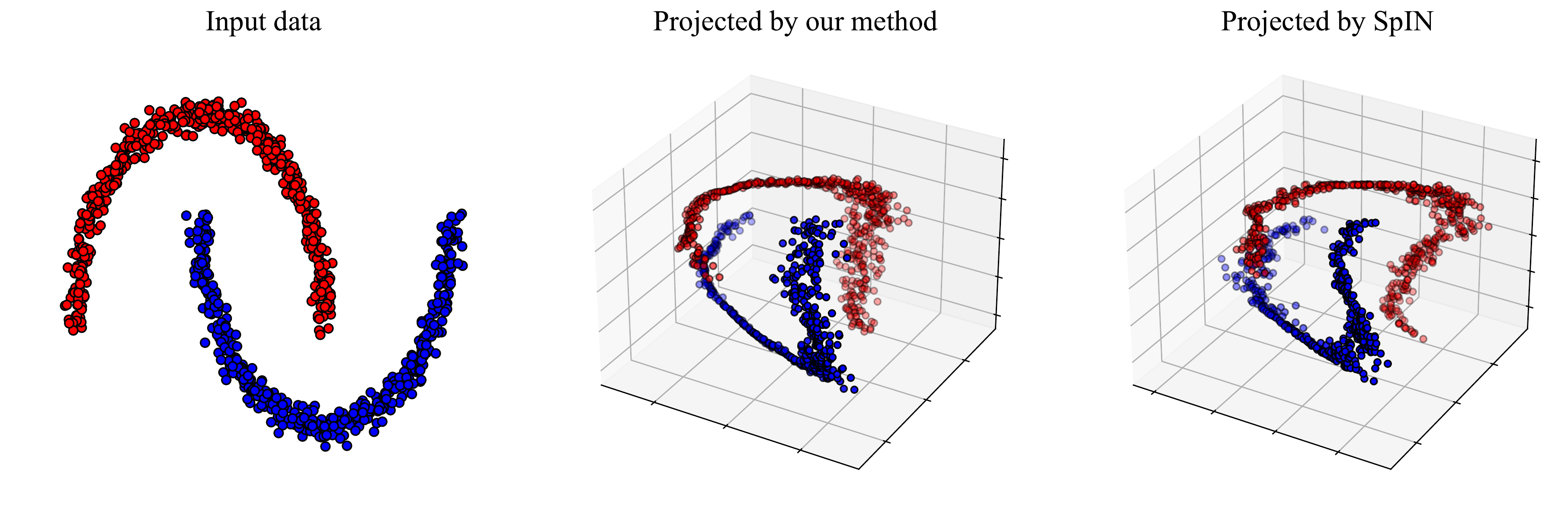}
    \vspace{-3ex}
    \caption{\scriptsize  ``Two-moon'' data}
    \label{fig:nngp-tm}
\end{subfigure}
\begin{subfigure}[b]{0.49\linewidth}
\centering
\includegraphics[width=\linewidth]{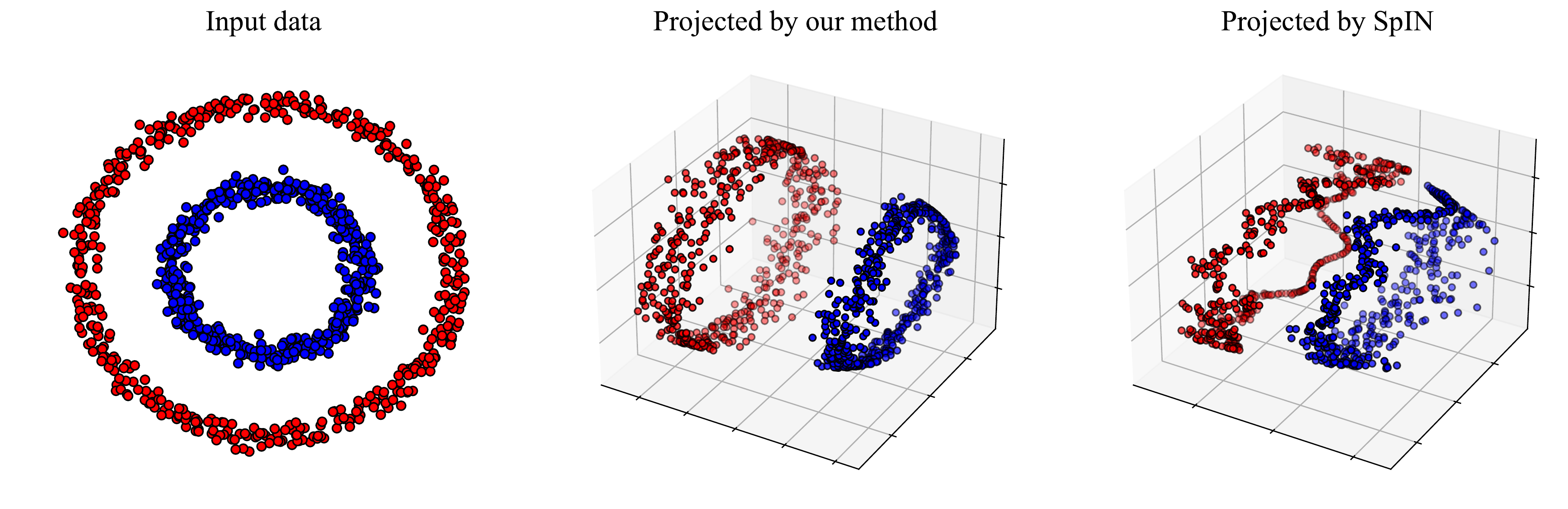}
    \vspace{-3ex}
    \caption{\scriptsize ``Circles'' data}
    \label{fig:nngp-circle}
\end{subfigure}
\caption{\footnotesize Project 2-D data to the 3-D space by the approximate top-$3$ eigenfunctions of the MLP-GP kernels. We use a 3-layer MLP with ReLU activations and a 3-layer MLP with Erf activations to specify the MLP-GP kernels for ``two-moon'' and ``circles'' respectively.}%
\label{fig:nngp}
\end{figure*}

\subsection{Learning with NN-GP kernels and NTKs}
\label{sec:scaleup}
The Nystr\"{o}m method, SpIN, and NeuralEF all entail computing the kernel matrix on training data $\vect{\kappa}^{\mathbf{X}_\text{tr},\mathbf{X}_\text{tr}}$, which may be frustratingly difficult for the NN-GP kernels and NTKs. %
As a workaround, we suggest approximately computing $\vect{\kappa}^{\mathbf{X}_\text{tr},\mathbf{X}_\text{tr}}$ with MC estimation based on \emph{finitely wide} NNs for NN-GP kernels and NTKs instead of performing analytical evaluation (see also \citet{novak2018bayesian}).

Specifically, for the NN-GP kernels, it is easy to see:
\begin{equation}
\small
\vect{\kappa}_\text{NN-GP}^{\mathbf{X}_\text{tr},\mathbf{X}_\text{tr}}  \approx \frac{1}{S}\sum_s^S g(\mathbf{X}_\text{tr}, \vtheta_s)g(\mathbf{X}_\text{tr}, \vtheta_s)^\top, %
\end{equation}
with $g(\mathbf{X}_\text{tr}, \vtheta_s) \in \mathbb{R}^{BN_\text{out}}$ as the concatenation of the vectorized outputs of $g$ and $\vtheta_s \sim p(\vtheta), s \in [S]$. 
In other words, we evaluate a finitely wide NN $g$ on the training data $\mathbf{X}_\text{tr}$ for $S$ times under various weight configurations to calculate the NN-GP training kernel matrix.

For the NTKs, it is hard to explicitly calculate and store the Jacobian matrix for modern NNs. 
Nevertheless, we notice that if there exists a distribution $p(\vv)$ satisfying $\mathbb{E}_{p(\vv)}[\vv\vv^\top] = \mathbf{I}_{\text{dim}(\vtheta)}$\footnote{$\text{dim}(\vtheta)$ denotes the dimensionality of $\vtheta$.}, then
\begin{equation}
\label{eq:ntk-matrix}
\small
\vect{\kappa}_\text{NTK}^{\mathbf{X}_\text{tr},\mathbf{X}_\text{tr}}  = \mathbb{E}_{\vv \sim p(\vv)}\left[\partial_{\vtheta}g(\mathbf{X}_\text{tr}, \vtheta) \vv\right] \left[\partial_{\vtheta}g(\mathbf{X}_\text{tr}, \vtheta) \vv\right]^\top. %
\end{equation}
By Taylor series expansion, we have $\partial_{\vtheta}g(\mathbf{X}_\text{tr}, \vtheta) \vv \approx \left({g(\mathbf{X}_\text{tr}, \vtheta + \epsilon\vv) - g(\mathbf{X}_\text{tr}, \vtheta)}\right)/{\epsilon}$ with $\epsilon$ as a small scalar (e.g., $10^{-5}$).
Combining this with MC estimation, we get
\begin{equation}
\small 
\Scale[0.95]{
\vect{\kappa}_\text{NTK}^{\mathbf{X}_\text{tr},\mathbf{X}_\text{tr}} \approx \frac{1}{S}\sum_{s=1}^S\left[\frac{g(\mathbf{X}_\text{tr}, \vtheta + \epsilon\vv_s) - g(\mathbf{X}_\text{tr}, \vtheta)}{\epsilon}\right]\left[{\frac{g(\mathbf{X}_\text{tr}, \vtheta + \epsilon\vv_s) - g(\mathbf{X}_\text{tr}, \vtheta)}{\epsilon}}\right]^\top} \nonumber
\end{equation}
with $\vv_s \sim p(\vv), s\in[S]$.
Namely, we perform forward propagation for $S+1$ times with the NN $g$ to approximately estimate the NTK training kernel matrix. 
Popular examples of $p(\vv)$ fulfilling the above requirement include the multivariate standard normal $\mathcal{N}(0, \mathbf{I}_{\text{dim}(\vtheta)})$ and the multivariate Rademacher distribution (the uniform distribution over $\{\pm 1\}^{\text{dim}(\vtheta)}$). 
We use the latter in our experiments.

\textbf{Discussion} One may question that now that we can approximate the NN-GP kernels and NTKs via these \emph{random feature} strategies, why do we still need to train a NeuralEF for kernel approximation? 
We clarify that the sample size $S$ for these strategies has to be large enough to realize high-fidelity approximation (e.g., $>1000$), which is unaffordable when testing. 
In contrast, NeuralEF captures the kernel by only several top eigenfunctions, which prominently accelerates the testing (see~\cref{sec:exp-ntk} for empirical evidence).

\section{Experiments}
We apply NeuralEF to a handful of kernel approximation scenarios to demonstrate its promise.
We set batch size $B$ as $256$ and optimize with an Adam~\cite{kingma2014adam} optimizer with $10^{-3}$ learning rate unless otherwise stated. 
The results of SpIN are obtained based on its official code\footnote{\url{https://github.com/deepmind/spectral_inference_networks}.}.
The detailed experiment settings are in \cref{app:exp}.

\subsection{Deconstructing Classic Kernels}
\label{sec:exp-classic}
At first, we experiment on classic kernels. %
The baselines include SpIN and the Nystr\"{o}m method. 
The latter converges to the ground truth solutions as the number of samples increases.
We set $q(x)$ as uniform distributions over $[-1, 1]$ and $[-2, 2]$ for the polynomial kernel $\kappa(x,x')=(x^\top x' + 1.5)^4$ and the RBF kernel $\kappa(x,x')=\exp(-\Vert x - x'\Vert^2 / 2)$, respectively. 
We use $k=10$ multilayer perceptrons (MLPs) of width $32$ to instantiate $\hat{\psi}_j$.
We set batch size $B$ as the dataset size $N$ if $N < 256$ else $256$, and train for $2000$ iterations.

\cref{fig:illu} and \cref{app:classic} display the recovered eigenfunctions under various scales of training data. 
As shown, NeuralEF delivers as good solutions as the Nystr\"{o}m method.
Yet, as more training samples involved, \emph{the training overhead of NeuralEF grows gracefully while the Nystr\"{o}m method suffers from inefficiency issues}.
The results also corroborate the inefficiency issues of SpIN.

\subsection{Deconstructing Modern Kernels}
We then use NeuralEF to deconstruct the NN-GP kernels and the NTKs specified by NNs with a scalar output.%
\subsubsection{The NN-GP Kernels}
\label{sec:exp-nngp}
\textbf{The MLP-GP kernels} We start from a simple problem of processing $1000$ 2-D data from the ``two-moon'' or ``circles'' dataset. 
We consider the NN-GP kernel associated with a 3-layer MLP architecture with ReLU or Erf activation, referred to as MLP-GP kernel for short. 
We utilize the strategies in \cref{sec:scaleup} to comfortably train NeuralEF. 
We are interested in the top-$3$ eigenfunctions for data visualization.
We instantiate the neural eigenfunctions as 3-layer MLPs of width $32$. 
The optimization settings are matched with those in
\cref{sec:exp-classic}.
We train SpIN models under the same settings for a fair comparison. 
After training, we use the approximate eigenfunctions to project the data into the 3-D space for visualization. %
As shown in \cref{fig:nngp}, NeuralEF yields more appealing outcomes than SpIN.

\begin{table}[t]
\centering
\footnotesize
\caption{\footnotesize 
The accuracy of Logistic regression classifiers trained 
on the projections of MNIST images given by the top-$10$ eigenfunctions for various kernels. 
Note that Nystr\"om method does not apply to our CNN-GP kernel because no analytical form of the kernel is known when max-pooling layers are used.
We have tuned the hyper-parameters for the polynomial and RBF kernels extensively but have not tuned those for the CNN-GP kernel.
}
  \begin{tabular}{
    >{\raggedright\arraybackslash}p{36ex}%
    >{\raggedleft\arraybackslash}p{16ex}%
    }%
  \toprule
\multicolumn{1}{c}{Method} & \multicolumn{1}{c}{LR test accuracy} \\% 
\midrule
\emph{Our method (CNN-GP kernel)}   & \textbf{84.98}\% \\
\emph{Nystr\"{o}m (CNN-GP kernel)}  & N/A \\
\emph{Nystr\"{o}m (polynomial kernel)}  & 78.00\%\\
\emph{Nystr\"{o}m (RBF kernel)}  & 77.55\%\\
  \bottomrule
   \end{tabular}
\label{table:1}
\end{table}

\begin{figure*}[t]
\centering
\includegraphics[width=0.75\linewidth]{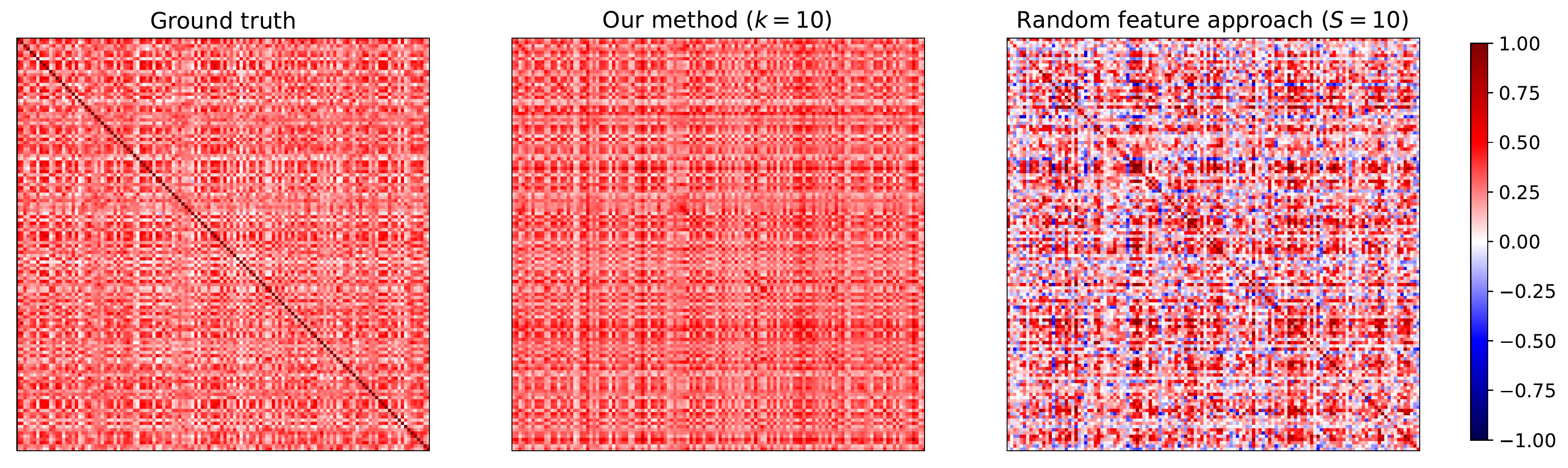}
\caption{\footnotesize %
NeuralEF %
v.s. random features %
for approximating the NTK %
of a ResNet-20 binary classifier. 
The evaluation is performed on $128$ CIFAR-10 test images. 
We explicitly compute the Jacobian matrix to attain the ground truth kernel matrix. 
Testing a datum with NeuralEF entails $k=10$ NNs forward passes. 
We use $S=10$ MC samples in the random feature approach %
such that the two methods have similar overhead.
}
\label{fig:ntk-b}
\vspace{-1ex}
\end{figure*}

\textbf{The CNN-GP kernels} We then consider using the NN-GP kernels specified by convolutional neural networks (CNNs) (dubbed as CNN-GP kernels) to process MNIST images. 
Without loss of generality, we use the following architecture \texttt{Conv5-ReLU-MaxPool2-Conv5-ReLU-MaxPool2} \texttt{-Linear-ReLU-Linear} with \texttt{Conv5} being a $5\times 5$ convolution and \texttt{MaxPool2} being a $2\times 2$ max-pooling. 
Note that the max-pooling operations make the analytical kernel evaluation 
highly nontrivial and no closed-form solution is known~\citep{novak2019neural}, thereby preventing the use of Nystr\"om method. 
Yet, NeuralEF is applicable as it does not hinge on kernelized solutions and the kernel matrix on training data can be efficiently estimated. 
We use $10$ CNNs also with the aforementioned CNN architecture to set up $\Scale[0.9]{\{\hat{\psi}_j\}_{j=1}^{10}}$, and train them on MNIST training images for $20000$ iterations.
After that, we use them to project all MNIST images to the 10-D space. 
We report the test accuracy of the (multiclass) Logistic regression (LR) model trained on these projections in \cref{table:1}.
We compare the results with those obtained by polynomial kernel and RBF kernel using Nystr\"{o}m method. %
We exclude SpIN from the comparison due to nontrivial implementation challenges involved in the application of SpIN to the CNN-GP kernel.  

 \cref{table:1} shows that the CNN-GP kernels are superior to the %
 classic kernels for processing images, and our approach can easily unleash their potential. 
This also %
demonstrates the efficacy of NeuralEF for dealing with \emph{out-of-sample extension}. 
We plot the top-$3$ dimensions of the projections given by NeuralEF in \cref{app:mnist-results}.
We see the projections form class-conditional clusters, implying that NeuralEF can learn the discriminative structures in CNN-GP kernels.

\begin{figure}[t]
\centering
\includegraphics[width=0.7\linewidth]{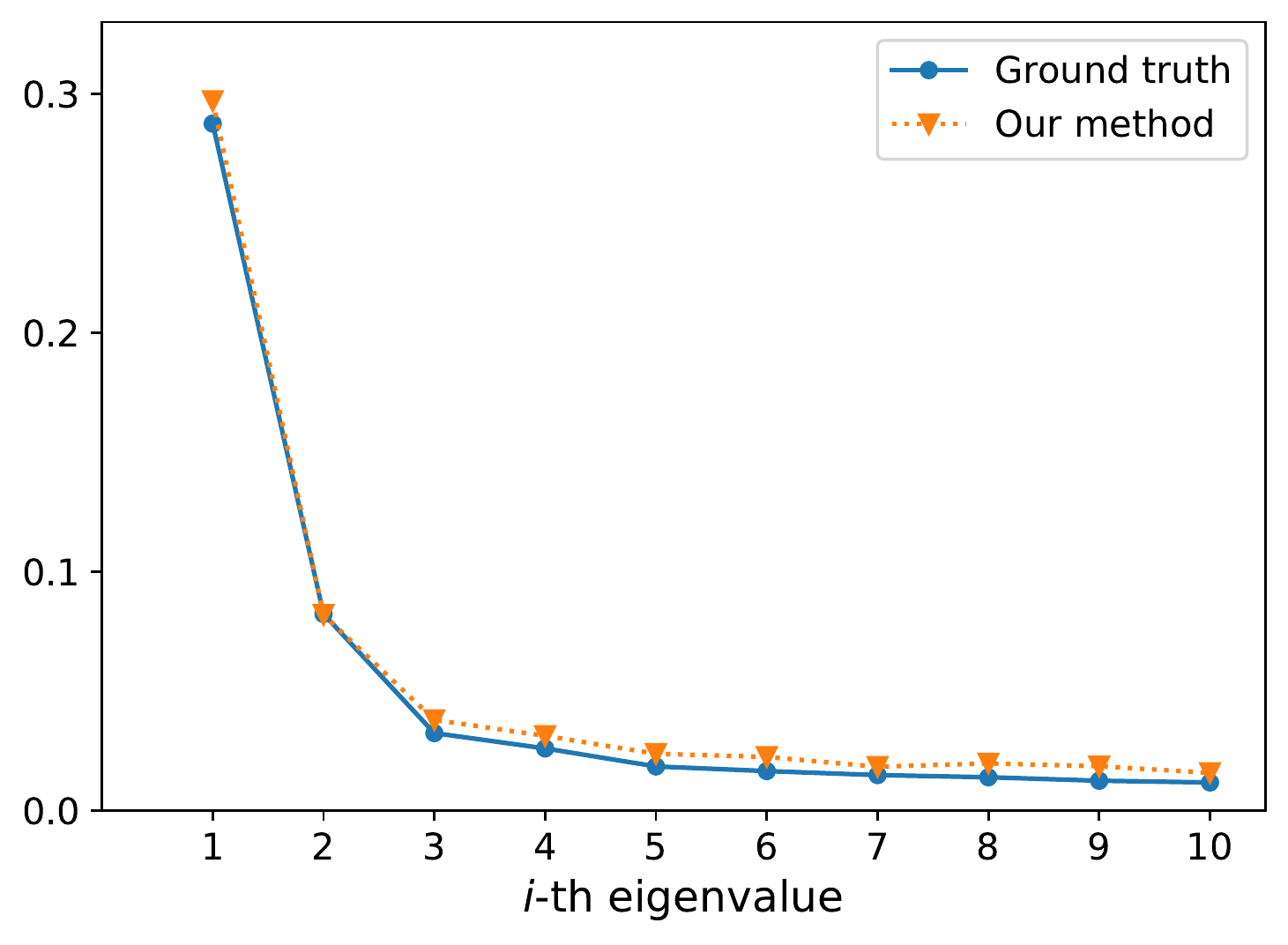}
\vspace{-1ex}
\caption{\footnotesize 
A plot of the eigenvalue  against the index of the eigenvalue for the NTK corresponding to a ResNet-20 binary classifier. 
Our method %
can approximate the top-$10$ ground truth eigenvalues with small errors.
}
\label{fig:ntk-a}
\end{figure}

\subsubsection{The NTKs}
\label{sec:exp-ntk}
We next experiment on the empirical NTKs corresponding to \emph{practically sized} NNs. %
Without loss of generality, we train a ResNet-20 classifier~\cite{he2016deep} to distinguish the airplane images from the automobile ones from CIFAR-10~\cite{krizhevsky2009learning}, and target the NTK associated with the trained binary classifier.
The MC estimation strategy in \cref{sec:scaleup} is used to approximate the training kernel matrix. 
We train a NeuralEF to find the top-$10$ eigenpairs of the NTK. 
Since analytical evaluation of the NTK of ResNet-20 is very time-consuming, we do not include the Nystr\"{o}m method in this experiment.

\cref{fig:ntk-a} displays the discovered top-$10$ eigenvalues and the ground truth ones, which are estimated by naively eigendecomposing the ground truth training kernel matrix. 
Besides, we use the found eigenfunctions to recover the kernel matrix on test data: $\kappa(\vx, \vx') \approx \sum_{j=1}^k \hat{\mu}_j \hat{\psi}_j(\vx)\hat{\psi}_j(\vx').$
The recovered matrix and the ground truth calculated from exact Jacobian are displayed in \cref{fig:ntk-b}. 
We also plot the kernel matrix estimated by the random feature (MC estimation) strategy in \cref{sec:scaleup} with $S=10$ samples.
We can see that: (\RN{1}) NeuralEF %
can approximate the top-$10$ ground truth eigenvalues with small errors; %
(\RN{2}) The kernel matrix recovered by NeuralEF suffers from distortion, perhaps due to that the eigenspectrum of the NTK has a \emph{long tail} and hence capturing only the top-$10$ eigenpairs is insufficient; %
(\RN{3}) Under identical resource consumption, the random feature approach underperforms NeuralEF in aspect of approximation quality, reflecting the importance of capturing eigenpairs; (\RN{4}) This NTK has \emph{low-rank} structures.

Further, we perform K-means clustering on the projections yielded by NeuralEF for the test images, and get \textbf{97.5}\% accuracy. 
As references, the K-means given the random features ($S=10$) yields 81.3\% accuracy, and the K-means given the raw pixels delivers 67.8\% accuracy.
This substantiates that \emph{NTK can characterize the intrinsic similarities between data points} and NeuralEF can approximate NTK in a more efficient way than random features.%

\begin{figure*}[t]
\centering
\begin{subfigure}[b]{0.19\linewidth}
\centering
\includegraphics[width=\linewidth]{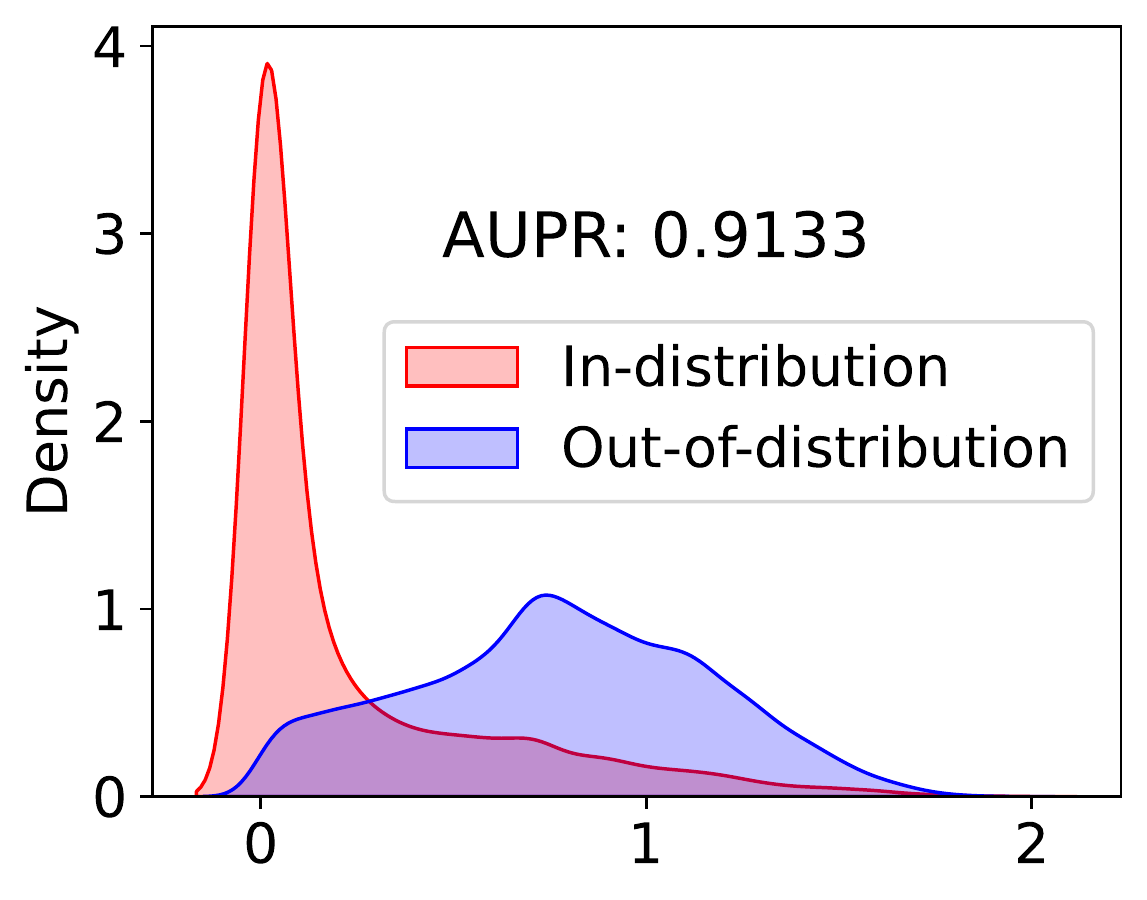}
    \vspace{-3.ex}
    \caption{\scriptsize Ours}
\end{subfigure}
\begin{subfigure}[b]{0.19\linewidth}
\centering
\includegraphics[width=\linewidth]{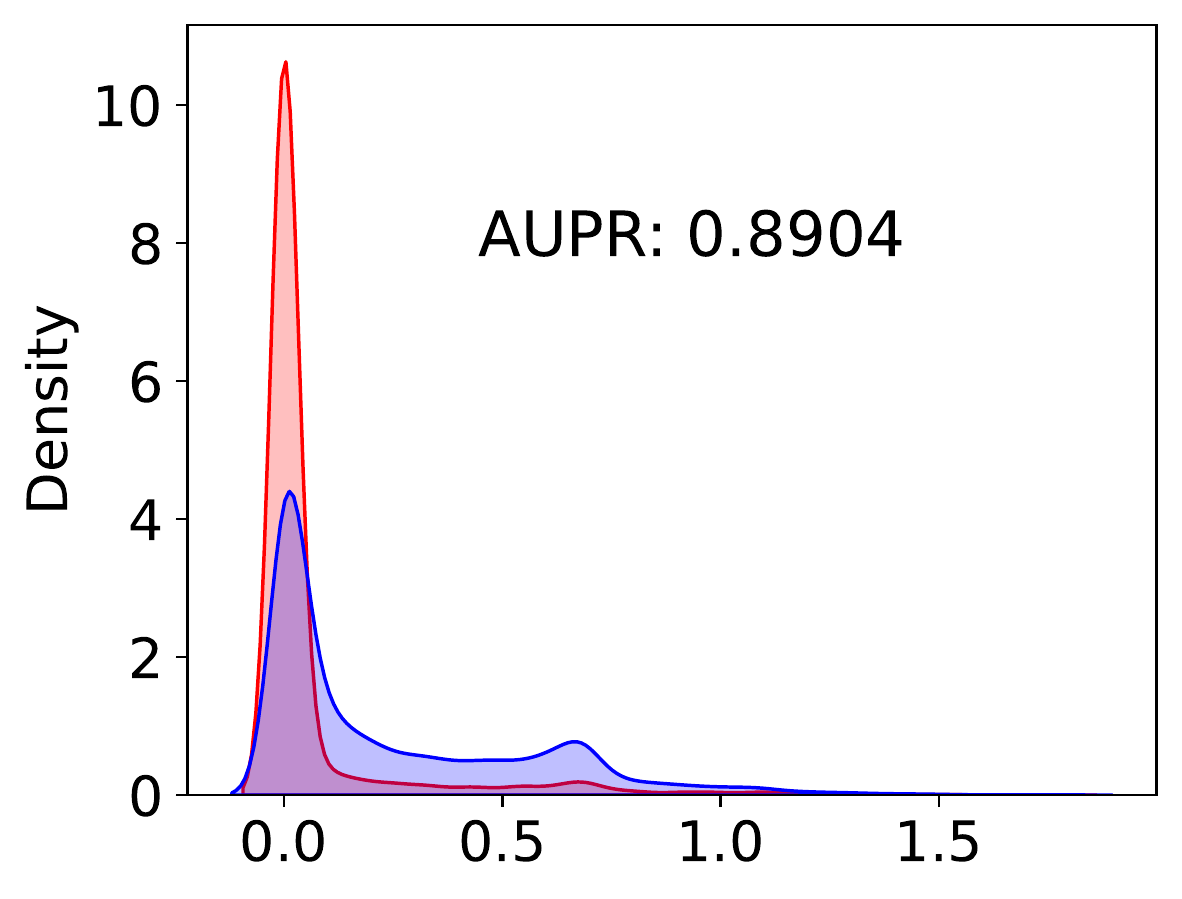}
    \vspace{-3.ex}
    \caption{\scriptsize MAP}
\end{subfigure}
\begin{subfigure}[b]{0.19\linewidth}
\centering
\includegraphics[width=\linewidth]{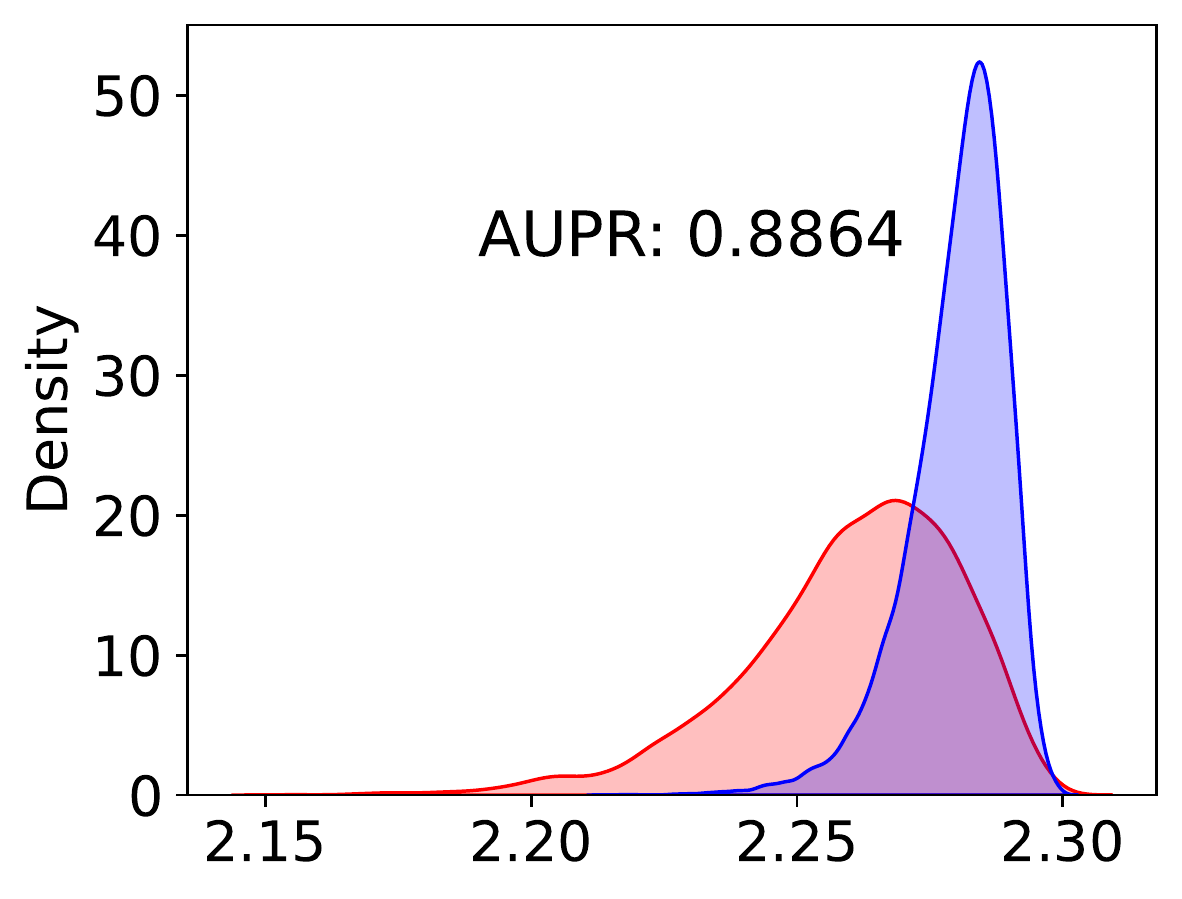}
    \vspace{-3.ex}
    \caption{\scriptsize KFAC LLA}
\end{subfigure}
\begin{subfigure}[b]{0.19\linewidth}
\centering
\includegraphics[width=\linewidth]{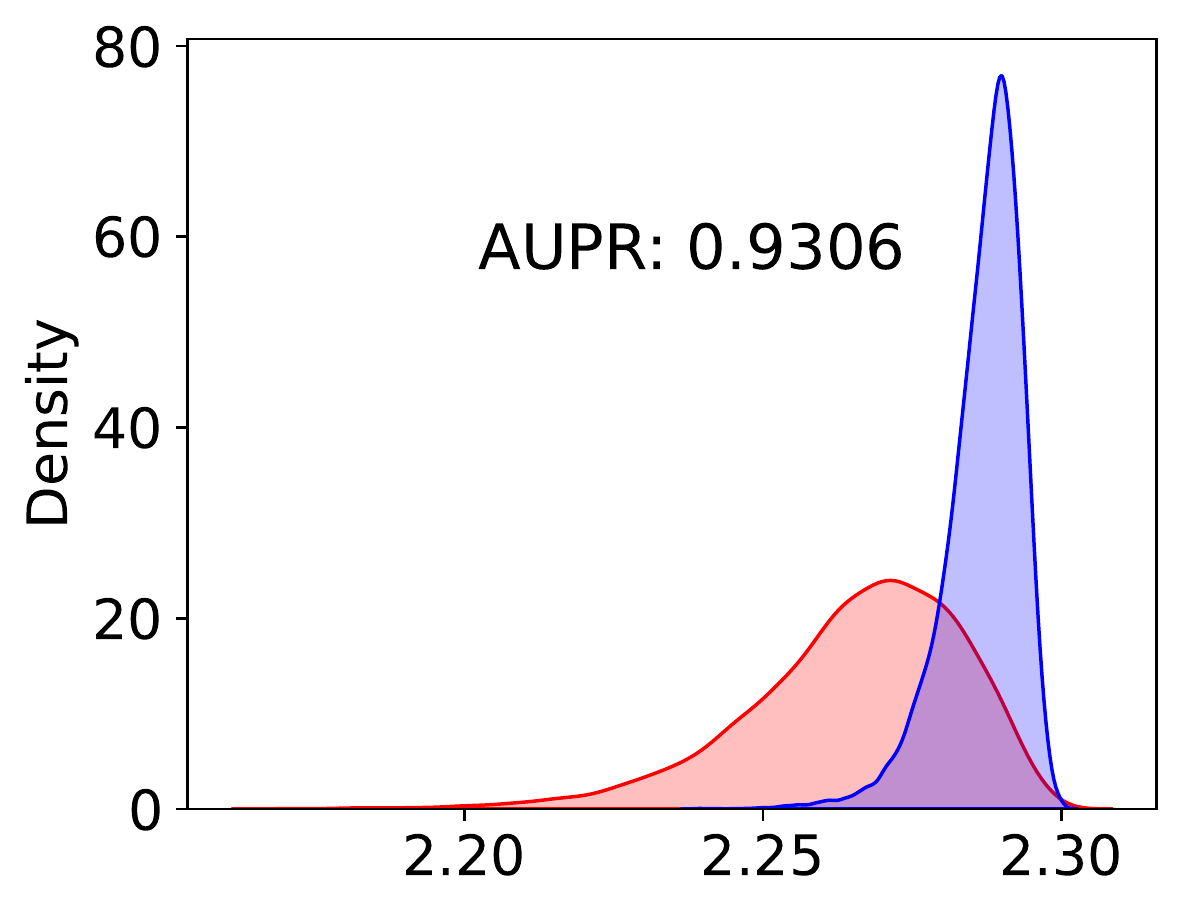}
    \vspace{-3.ex}
    \caption{\scriptsize Diag LLA}
\end{subfigure}
\begin{subfigure}[b]{0.19\linewidth}
\centering
\includegraphics[width=\linewidth]{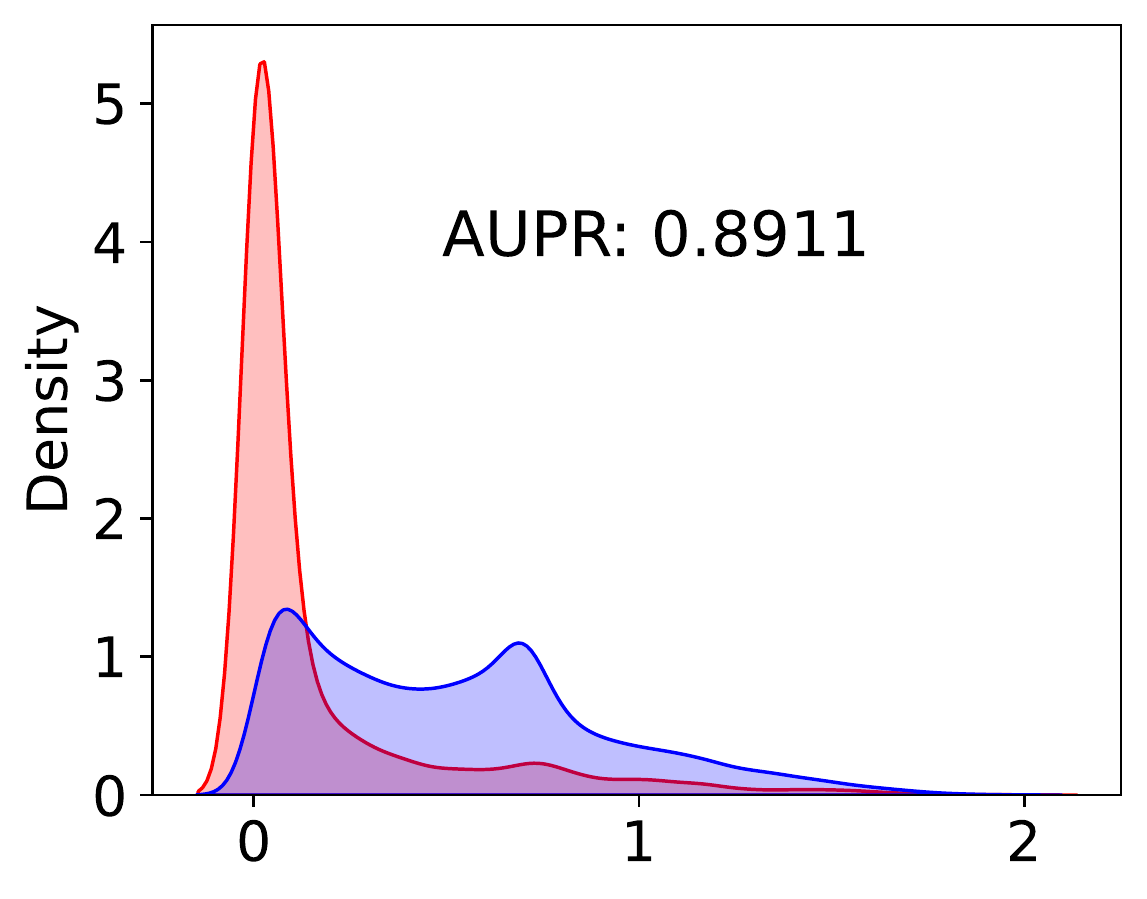}
    \vspace{-3.ex}
    \caption{\scriptsize Last-layer LLA}
\end{subfigure}
\vspace{-1ex}
\caption{\footnotesize The histograms of the uncertainty estimates for in-distribution CIFAR-10 test data and OOD SVHN test data. 
We experiment on ResNet-110 here. 
AUPR refers to the area under the precision-recall curve.}
\label{fig:ntkunc}
\end{figure*}

\subsection{Advanced Application: Scale up Linearised Laplace Approximation}
\label{sec:ntk}
Laplace approximation (LA)~\cite{mackay1992bayesian} is a canonical approach for approximate posterior inference, yet suffers from underfitting~\cite{lawrence2001variational}.
\citet{foong2019between} proposed to linearise the output of the model about the maximum a posteriori (MAP) solution to alleviate this issue (see also \cite{khan2019approximate}).
Specifically, we consider learning an NN model $g(\cdot, \vtheta): \mathcal{X}\rightarrow N_\text{out}$ for data  $\{(\vx_i, \vy_i)\}_{i=1}^N$ under an isotropic Gaussian prior $p(\vtheta)=\mathcal{N}(0, \sigma_0^2\mathbf{I}_{\text{dim}(\vtheta)})$.
The linearised Laplace approximation (LLA) first finds the maximum a posteriori (MAP) solution $\vtheta_\text{MAP}$, and then uses the following Gaussian process to form a function-space approximate posterior:
\begin{equation}
\small
\label{eq:la-naive}
\mathcal{GP}(f | g(\vx, \vtheta_\text{MAP}), \partial_{\vtheta}g(\vx, \vtheta_\text{MAP}) \mathbf{\Sigma}\partial_{\vtheta}g(\vx', \vtheta_\text{MAP})^\top), %
\end{equation}
where $\mathbf{\Sigma}$ is the inversion of 
the Gauss-Newton matrix\footnote{The Gauss-Newton matrix serves as a workaround of the Hessian since that the Hessian is more difficult to estimate. }: $\Scale[0.95]{\mathbf{\Sigma}^{-1}=\sum_i \partial_{\vtheta}g(\vx_i, \vtheta_\text{MAP})^\top \mathbf{\Lambda}_i \partial_{\vtheta}g(\vx_i, \vtheta_\text{MAP}) + 1/\sigma_0^2\mathbf{I}_{\text{dim}(\vtheta)}}$ with $\Scale[0.95]{\mathbf{\Lambda}_i := -\partial_{\vg\vg}^2 \log p(\vy_i|\vg) |_{\vg=g(\vx_i, \vtheta_\text{MAP})} \in \mathbb{R}^{N_\text{out} \times N_\text{out}}}$.

However, like the vanilla Laplace approximation, LLA also suffers from the time-consuming matrix inversion on the matrix of size $\text{dim}(\vtheta) \times \text{dim}(\vtheta)$. 
We next show that NeuralEF can be incorporated to overcome this issue.

The covariance kernel in \cref{eq:la-naive} is closely related to the NTK $\kappa_\text{NTK}(\vx, \vx')=\partial_{\vtheta}g(\vx, \vtheta_\text{MAP}) \partial_{\vtheta}g(\vx', \vtheta_\text{MAP})^\top$. 
Empowered by this observation, we prove that the above GP can be approximated as (see \cref{app:proof-la-our} for the proof):
\begin{equation}
\small
\label{eq:la-our}
\Scale[0.9]{\mathcal{GP}\left(f | g(\vx, \vtheta_\text{MAP}), \tilde{\Psi}(\vx)  \left[\sum_i \tilde{\Psi}(\vx_i)^\top \mathbf{\Lambda}_i \tilde{\Psi}(\vx_i) + \frac{1}{\sigma_0^2}\mathbf{I}_k\right]^{-1} \tilde{\Psi}(\vx')^\top\right)}, %
\end{equation}
where $\Scale[0.95]{\tilde{\Psi}(\vx):=[\sqrt{\hat{\mu}_1}\hat{\psi}_1(\vx), ..., \sqrt{\hat{\mu}_k}\hat{\psi}_k(\vx)] \in \mathbb{R}^{N_\text{out} \times k}}$ with $\Scale[0.95]{\hat{\psi}_i: \mathcal{X}\rightarrow \mathbb{R}^{N_\text{out}}}$ as the approximate multi-output eigenfunction (see \cref{sec:method}) corresponding to the approximate $i$-th largest eigenvalues $\hat{\mu}_i$ of $\kappa_\text{NTK}$. 
With this, we only need to invert a matrix of size $k \times k$ to estimate the covariance.

Then, the whole pipeline for the refined LLA is: we first find $\vtheta_\text{MAP}$, then use NeuralEF to find the top-$k$ (we set $k=10$ in the experiments) eigenpairs of the kernel $\Scale[0.9]{\kappa_\text{NTK}(\vx, \vx')=\partial_{\vtheta}g(\vx, \vtheta_\text{MAP}) \partial_{\vtheta}g(\vx', \vtheta_\text{MAP})^\top}$, then iterate through the training set to compute $\Scale[0.9]{\sum_i \tilde{\Psi}(\vx_i)^\top \mathbf{\Lambda}_i \tilde{\Psi}(\vx_i)}$, and finally obtain the approximate posterior in \cref{eq:la-our}.

We verify on CIFAR-10 using ResNet architectures where $\text{dim}(\vtheta) > 10^{5}$.
Naive LLA fails due to scalability issues, so we take MAP, Kronecker factored LLA (KFAC LLA), diagonal LLA (Diag LLA), and last-layer LLA as baselines. 
We implement the last three baselines based on the \texttt{laplace} library\footnote{\url{https://github.com/AlexImmer/Laplace}.}~\cite{daxberger2021laplace}. 
The training kernel matrix for NeuralEF is of size $5\cdot 10^5 \times 5\cdot 10^5$, but the training only lasts for half a day. 
As these methods exhibit similar test accuracy, we only report the comparison on negative log-likelihood (NLL) and expected calibration error (ECE)~\cite{guo2017calibration} in \cref{table:3}. 
We also depict the histograms of the uncertainty estimates (measured by predictive entropy) for in-distribution and out-of-distribution (OOD) data in \cref{fig:ntkunc}.
As shown, LLA with NeuralEF is consistently better than the baselines in the aspect of model calibration. 
The NLLs and predictive uncertainty of LLA with NeuralEF are also better than or on par with the competitors.

See \cref{sec:sgd} for one more application of NeuralEF where we use it to approximate the implicit kernel induced by the stochastic gradient descent (SGD) trajectory to perform Bayesian deep learning (BDL).

\begin{table}[t]
  \centering
 \footnotesize
 \caption{\small Comparison on test NLL $\downarrow$ and ECE $\downarrow$ on CIFAR-10.}
  \label{table:3}
  \setlength{\tabcolsep}{5.5pt}
  \begin{tabular}{
   >{\raggedright\arraybackslash}p{13.5ex}%
   >{\raggedleft\arraybackslash}p{4ex}%
   >{\raggedleft\arraybackslash}p{4ex}%
   >{\raggedleft\arraybackslash}p{4ex}%
   >{\raggedleft\arraybackslash}p{4ex}%
   >{\raggedleft\arraybackslash}p{4ex}%
   >{\raggedleft\arraybackslash}p{4ex}%
  }
  \toprule
\multicolumn{1}{c}{\multirow{2}{*}{Method}}& \multicolumn{2}{c}{ResNet-20} &\multicolumn{2}{c}{ResNet-56} &\multicolumn{2}{c}{ResNet-110}\\ 
& \multicolumn{1}{c}{NLL} & \multicolumn{1}{c}{ECE}& \multicolumn{1}{c}{NLL} & \multicolumn{1}{c}{ECE}& \multicolumn{1}{c}{NLL} & \multicolumn{1}{c}{ECE}\\% &NLL & ECE\\
\midrule
\emph{Ours} & \textbf{0.277} & \textbf{0.016}  & \textbf{0.234} & \textbf{0.012} & \textbf{0.241} & \textbf{0.010} \\ %
\emph{MAP} & 0.357 & 0.049  & 0.336 & 0.050 & 0.345 & 0.046\\ %
\emph{KFAC LLA} & 0.906 & 0.468  & 1.576 & 0.707 & 1.767 & 0.749\\ 
\emph{Diag LLA} & 0.934 & 0.480  & 1.606 & 0.712 & 1.797 & 0.754\\ 
\emph{Last-layer LLA} & \textbf{0.264} & 0.026  & \textbf{0.231} & 0.024 & \textbf{0.233} & 0.019\\ 

  \bottomrule
   \end{tabular}
  
\end{table}

\section{Related Work}
Recently, there is ongoing effort to associate specific (Bayesian) NNs with kernel methods or Gaussian processes to gain insights for the theoretical understanding of NNs~\cite{neal1996priors,lee2017deep,garriga2018deep,matthews2018gaussian,novak2018bayesian,jacot2018neural,arora2019exact,khan2019approximate,sun2020neural} or to enrich the family of kernels~\cite{wilson2016deep}.
However, it has been rarely explored how to solve the equally important ``inverse'' problem---designing appropriate NN counterparts for the kernels of interest.
We show in this work that approximating kernels with NNs can be the key to scaling up kernel methods to large data.

{The Nystr\"{o}m method}~\cite{nystrom1930praktische,williams2001using} is a classic kernel approximation method, and has been extended to enable the out-of-sample extension of spectral embedding methods by \citet{bengio2004learning}. 
But as discussed, the Nystr\"{o}m method faces inefficiency issues when handling big data and modern kernels like NTKs. 
Instead, deconstructing kernels by NNs has the potential to ameliorate these pathologies due to the expressiveness and scalability of NNs. 
SpIN is the first work in this spirit~\cite{pfau2018spectral}.
It trains NNs to approximate the top eigenfunctions of the kernel for kernel approximation, but it suffers from an ill-posed objective function
and thereby an %
involved learning procedure.
It is difficult to extend SpIN to treat modern kernels and big data due to the the requirement of Cholesky decomposition and manipulation of Jacobians. 
Relatedly, 
EigenGame~\cite{gemp2020eigengame} identifies a common mistake in literature for interpreting PCA as an optimization problem and proposes ways to fix it. 
It turns out that the same spirit also applies to fixing the SpIN objective function.
In fact, NeuralEF can be viewed as a function-space extension of EigenGame.

\section{Conclusion}
We propose NeuralEF for scalable kernel approximation in this paper. 
During the derivation of the method, we have deepened the connections between kernels and NNs. 
We show the efficacy of NeuralEF and further apply it in several interesting yet challenging scenarios in unsupervised and supervised learning. 
We discuss the limitations and possible future works of NeuralEF below.

\textbf{Limitations} Currently, we represent each eigenfunction with a dedicated NN, so we train $k$ NNs to cover the top-$k$ eigenpairs. This may become costly when we have to use a large $k$ (i.e., the eigenspectrum is long-tail). %
Besides, we empirically observe that NeuralEF has difficulties %
capturing the eigenpairs with relatively small  eigenvalues (e.g., 1\% of the largest eigenvalue) perhaps due to issues in stochastic optimization or %
numerical errors. 
Finally, It is difficult to tune the hyper-parameters of the kernel while using NeuralEF approximations.

\textbf{Future work} To promote parameter efficiency and potentially support a very large $k$, we need to perform \emph{weight-sharing} among the $k$ neural eigenfunctions.
More appealing applications such as using NeuralEF as unsupervised representation learners also deserve future investigation.

\section*{Acknowledgments}
This work was supported by the National Key Research and Development Program of China (Nos. 2020AAA0104304,  2017YFA0700904), NSFC Projects ((Nos. 62061136001, 62106122, 62076147, U19B2034, U1811461, U19A2081)), Tsinghua-Huawei Joint Research Program, a grant from Tsinghua Institute for Guo Qiang, and the High Performance Computing Center, Tsinghua University.

\nocite{langley00}

\bibliography{example_paper}
\bibliographystyle{icml2022}

\newpage
\appendix
\onecolumn
\section{Proof}
\subsection{Proof of \cref{theorem:0}}\label{app:proof}
\Thm*
\begin{proof}

First of all, it is easy to see that the functions in $L^2(\mathcal{X}, q)$ form an inner product space with inner product given by:
\begin{equation*}
\small
\langle \varphi, \varphi' \rangle = \int \varphi(\vx)\varphi'(\vx)q(\vx)d\vx, \;\; \forall \varphi, \varphi' \in L^2(\mathcal{X}, q).
\end{equation*}

By Mercer's theorem, we have $\kappa(\vx, \vx') = \sum_{j\geq1}\mu_j \psi_j(\vx)\psi_j(\vx')$ where $(\mu_j, \psi_j)$ refers to the $j$-th ground truth eigenpair of $\kappa$. Without loss of generality, we assume $\forall i < j$: $\mu_i > \mu_j$.

It is easy to reformulate the problems in \cref{eq:obj} as:
\begin{equation*}
\small
\begin{aligned}
\max_{\hat{\psi}_1} R_{11} &\;\; \text{s.t.:}\, C_1 = 1 \\
\max_{\hat{\psi}_2} R_{22} &\;\; \text{s.t.:}\, C_2 = 1, R_{12} = 0 \\
\max_{\hat{\psi}_3} R_{33} &\;\; \text{s.t.:}\, C_3 = 1, R_{13} = 0, R_{23} = 0\\
&...
\end{aligned}
\end{equation*}
Interestingly, when simultaneously solving these problems, the solution to the $j$-th problem only depends on the those to the preceding problems and is independent of those to the following problems. 
Therefore, we can derive the solutions to the problems in a sequencial manner.

Specifically, we first consider the maximization objective in the first problem:
\begin{equation*}
\small
\begin{aligned}
R_{11} &= \iint \hat{\psi}_1(\vx)\kappa(\vx, \vx')\hat{\psi}_1(\vx')q(\vx')q(\vx)d\vx'd\vx \\
&=\iint \hat{\psi}_1(\vx)\left(\sum_{j\geq1}\mu_j \psi_j(\vx)\psi_j(\vx')\right)\hat{\psi}_1(\vx')q(\vx')q(\vx)d\vx'd\vx \\
&=\sum_{j\geq1} \mu_j \iint \hat{\psi}_1(\vx) \psi_j(\vx)\psi_j(\vx')\hat{\psi}_1(\vx')q(\vx')q(\vx)d\vx'd\vx \\
&=\sum_{j\geq1} \mu_j \int\hat{\psi}_1(\vx) \psi_j(\vx)q(\vx)d\vx \int\hat{\psi}_1(\vx') \psi_j(\vx')q(\vx')d\vx'\\
&=\sum_{j\geq1} \mu_j \langle \hat{\psi}_1, \psi_j \rangle^2.
\end{aligned}
\end{equation*}

Given the definition that the eigenfunctions $\{\psi_j\}_{j\geq1}$ are orthonormal (see \cref{eq:eigenfunc-c}), we know they form a set of orthonormal bases of the $L^2(\mathcal{X}, q)$ space. 
Thus, we can represent $\hat{\psi}_1$ in such a new axis system by coordinate $(w_1, w_2,...)$:
\begin{equation*}
\small
\hat{\psi}_1 = \sum_{i\geq1} w_i\psi_i.
\end{equation*}

We can then rewrite the maximization objective as:
\begin{equation*}
\small
R_{11} = \sum_{j\geq1} \mu_j \langle \hat{\psi}_1, \psi_j \rangle^2 = \sum_{j\geq1} \mu_j \langle \sum_{i\geq1} w_i\psi_i, \psi_j \rangle^2 = \sum_{j\geq1} \mu_j w_j^2.
\end{equation*}

Recalling the constraint $C_1 = 1$, we have 
\begin{equation*}
\small
\langle \hat{\psi}_1, \hat{\psi}_1 \rangle = \langle \sum_{i\geq1} w_i\psi_i, \sum_{j\geq1} w_j\psi_j \rangle = \sum_{i,j\geq1} w_i w_j\langle \psi_i, \psi_j \rangle = \sum_{j\geq1} w_j^2 = 1.
\end{equation*}

Then, it is straight-forward to see the maximum value of $R_{11}$ is the largest ground truth eigenvalue $\mu_1$.
The condition to make the maximization hold is that $(w_1, w_2,...)$ is a one-hot vector with the first element as $1$, namely, $\hat{\psi}_1=\psi_1$.  
Thereby, we prove that solving the first problem uncovers the first eigenvalue as well as the associated eigenfunction of the kernel $\kappa$. 

We then consider solving the second problem given $\hat{\psi}_1=\psi_1$. 
Compared to the first problem, there is one more constraint:
\begin{align*}
\small 
R_{12} &= 0\\
\Rightarrow\;\; \iint \hat{\psi}_1(\vx)\kappa(\vx, \vx')\hat{\psi}_2&(\vx')q(\vx')q(\vx)d\vx'd\vx = 0 \\
\Rightarrow\;\; \int \hat{\psi}_2(\vx')q(\vx')\int \hat{\psi}_1&(\vx) \kappa(\vx, \vx')q(\vx)d\vx d\vx' = 0 \\
\Rightarrow\;\; \int \hat{\psi}_2(\vx')q(\vx')\int \psi_1&(\vx) \kappa(\vx, \vx')q(\vx)d\vx d\vx' = 0 \\%& \pushright{(\hat{\psi}_1^*=\psi_1)}\\
\Rightarrow\;\; \int \hat{\psi}_2(\vx')q(&\vx')\mu_1 \psi_1(\vx') d\vx' = 0 \\
\Rightarrow\;\; \langle \psi_1, &\hat{\psi}_2 \rangle = 0. \\
\end{align*}
Namely, $\hat{\psi}_2$ is constrained in the orthogonal complement of the subspace spanned by $\psi_1$. 
Given such a minor difference between the second problem and the first problem, we can apply an analysis similar to that for the first problem to solve the second problem. 
Note that $\mu_2$ is the largest eigenvalue in the orthogonal complement of the subspace spanned by $\psi_1$, $(R_{22}, \hat{\psi}_2)$ would hence converge to the $2$-th largest eigenvalue and the associated eigenfunction of $\kappa$.

Applying this procedure incrementally to the additional problems then finishes the proof. 
\end{proof}

\subsection{Justification of NeuralEF in Practice}
\label{app:analysis}
Though, ideally, $\hat{\psi}_j$ would converge to the ground truth
eigenfunction ${\psi}_j$ by \cref{theorem:0}, in practice, when performing optimization according to \cref{eq:obj-k} we cannot guarantee that at the convergence $\hat{\psi}_j$ is exactly equivalent to ${\psi}_j$. 
May the small deviation in the solutions to the preceding problems significantly bias the solutions to the following problems?
We show that it is not the case below.

To simplify the analysis, we consider only the first two problems in \cref{eq:obj-k}, namely, learning the first two approximate eigenfunctions $\hat{\psi}_1$ and $\hat{\psi}_2$.
As done in the proof of \cref{theorem:0}, we represent $\hat{\psi}_1$ and $\hat{\psi}_2$ in the axis system specified by $\{\psi_j\}_{j\geq1}$:
\begin{equation*}
\small
\hat{\psi}_1 = \sum_{i\geq1} v_i\psi_i \;\;\; \hat{\psi}_2 = \sum_{i\geq1} w_i\psi_i.
\end{equation*}

At first, we assume that  solving $\max_{\hat{\psi}_1} R_{11}\; \text{s.t.:}\, C_1=1$, whose optima is ${\psi}_1$, leads to that $\Vert \hat{\psi}_1 - {\psi}_1 \Vert^2 < 2\epsilon_1 < 2$, where $\Vert \cdot \Vert$ refers to the norm induced by the inner product in $L^2(\mathcal{X}, q)$. Then,
\begin{equation*}
\small
\begin{aligned}
\Vert \hat{\psi}_1 - {\psi}_1 \Vert^2 < 2\epsilon_1 \;\Rightarrow\; 2 - 2\langle\hat{\psi}_1, {\psi}_1\rangle < 2\epsilon_1 \;\Rightarrow\; 2 - 2v_1 < 2\epsilon_1 \;\Rightarrow\; v_1 > 1 - \epsilon_1. 
\end{aligned}
\end{equation*}

As $\hat{\psi}_1$ is normalized, we have:
\begin{equation*}
\small
\begin{aligned}
v_1^2 + v_2^2 \leq 1 \;\Rightarrow\; v_2^2 \leq 1 - v_1^2 \leq 1 - (1 - \epsilon_1)^2 = 2\epsilon_1 - \epsilon_1^2 \;\Rightarrow\; |v_2| < \sqrt{2\epsilon_1 - \epsilon_1^2}. 
\end{aligned}
\end{equation*}

We then assume that solving $\max_{\hat{\psi}_2} -\frac{R_{12}^2}{R_{11}}$ results in that $|R_{12}| < \epsilon_2 < 1$.
Given that
\begin{equation*}
\small
\begin{aligned}
R_{12} &= \iint \hat{\psi}_1(\vx)\kappa(\vx, \vx')\hat{\psi}_2(\vx')q(\vx')q(\vx)d\vx'd\vx \\
&= \iint \left(\sum_{i\geq 1} v_i {\psi}_i(\vx)\right) \left(\sum_{j\geq1}\mu_j \psi_j(\vx)\psi_j(\vx')\right) \hat{\psi}_2(\vx')q(\vx')q(\vx)d\vx'd\vx \\
&= \sum_{i\geq 1}v_i \sum_{j\geq 1}\mu_j \int {\psi}_i(\vx)\psi_j(\vx)q(\vx)d\vx \int \psi_j(\vx') \hat{\psi}_2(\vx')q(\vx')d\vx' \\
&= \sum_{i\geq 1}v_i \sum_{j\geq 1}\mu_j \mathbbm{1}[i=j] \langle\psi_j, \hat{\psi}_2\rangle \\
&= \sum_{i\geq 1}v_i \mu_i  \langle\psi_i, \hat{\psi}_2\rangle\\
&= \sum_{i\geq 1}v_i \mu_i  w_i, \\
\end{aligned}
\end{equation*}
we have $|\sum_{i\geq 1}v_i \mu_i  w_i|< \epsilon_2$.

We consider the optima $\hat{\psi}^*_2$ of the problem $\max_{\hat{\psi}_2} R_{22}-\frac{R_{12}^2}{R_{11}} \; \text{s.t.:}\, C_2=1$, which equals to 
\begin{equation*}
\small
\max_{\{w_j\}_{j\geq 1}} \sum_{j\geq1} \mu_j w_j^2 \;\text{s.t.:} \, |\sum_{i\geq 1}v_i \mu_i  w_i|< \epsilon_2,\, \sum_{j\geq 1} w_j^2 = 1.
\end{equation*}

It is easy to see at the maximum, $w_i^2 = 0, \, \forall \, i > 2$.
Namely, $\hat{\psi}^*_2=w_1 \psi_1 + w_2 \psi_2$ with $w_1^2 + w_2^2 = 1$.

Without loss of generality, we assume $w_2 > 0$ and  measure the distance between $\hat{\psi}^*_2$ and  $\psi_2$ to estimate the induced bias:\footnote{It $w_2 < 0$, we measure the distance between $\hat{\psi}^*_2$ and  $-\psi_2$, and the consequence is the same.}
\begin{equation*}
\small
\Vert \hat{\psi}^*_2 - {\psi}_2 \Vert^2 = 2 - 2\langle \hat{\psi}^*_2, {\psi}_2 \rangle = 2 - 2w_2 = 2 - 2\sqrt{1 - w_1^2} \leq 2 - 2(1- w_1^2) = 2 w_1^2.
\end{equation*}

Recall that $|v_1 \mu_1  w_1 + v_2\mu_2 w_2|< \epsilon_2$, $v_1 > 1 - \epsilon_1$, and $|v_2| < \sqrt{2\epsilon_1 - \epsilon_1^2}$, then 

1) when $w_1 > 0$:
\begin{equation*}
\small
(1-\epsilon_1)\mu_1 w_1 < v_1\mu_1w_1 < \epsilon_2 - v_2 \mu_2 w_2 \leq \epsilon_2 + \sqrt{2\epsilon_1 - \epsilon_1^2} \mu_2,
\end{equation*}
then $w_1 < \frac{1}{(1-\epsilon_1)\mu_1} (\epsilon_2 + \sqrt{2\epsilon_1 - \epsilon_1^2} \mu_2)$;

2) when $w_1 < 0$:
\begin{equation*}
\small
(1-\epsilon_1)\mu_1 w_1 > v_1\mu_1w_1 > - \epsilon_2 - v_2 \mu_2 w_2 \geq - \epsilon_2 - \sqrt{2\epsilon_1 - \epsilon_1^2} \mu_2,
\end{equation*}
then $w_1 > \frac{1}{(1-\epsilon_1)\mu_1} (-\epsilon_2 - \sqrt{2\epsilon_1 - \epsilon_1^2} \mu_2)$.

Therefore, $|w_1|$ is small if $\Vert \hat{\psi}_1 - {\psi}_1 \Vert^2$ and $|R_{12}|$ are small, and in turn $\Vert \hat{\psi}^*_2 - {\psi}_2 \Vert^2=2w_1^2$ is small. 

Applying this procedure incrementally to the additional problems then finishes the whole justification.

\subsection{Proof of \cref{eq:la-our}}
\label{app:proof-la-our}
\begin{proof}
We denote $\partial_{\vtheta}g(\vx, \vtheta_\text{MAP})$ as $\mathbf{J}^*_\vx$ for compactness. 
We concatenate $\{\mathbf{J}^*_{\vx_i} \in \mathbb{R}^{N_\text{out} \times \text{dim}(\vtheta)}\}_{i=1}^N$ as a big matrix $\mathbf{J}^*_{\mathbf{X}_\text{tr}} \in \mathbb{R}^{NN_\text{out} \times \text{dim}(\vtheta)}$, and organize $\{\mathbf{\Lambda}_i \in \mathbb{R}^{N_\text{out} \times N_\text{out}} \}_{i=1}^N$ as a block-diagonal matrix $\mathbf{\Lambda}_{\mathbf{X}_\text{tr}} \in \mathbb{R}^{NN_\text{out} \times NN_\text{out}}$. 
Then, by Woodbury matrix identity~\cite{woodbury1950inverting}, $\mathbf{\Sigma}$ can be rephrased:
\begin{equation*}
\begin{aligned}
\small
\mathbf{\Sigma}&=\left[\sum_i {\mathbf{J}^*_{\vx_i}}^\top \mathbf{\Lambda}_i \mathbf{J}^*_{\vx_i} + 1/\sigma_0^2\mathbf{I}_{\text{dim}(\vtheta)}\right]^{-1}\\
&= \left[{\mathbf{J}^*_{\mathbf{X}_\text{tr}}}^\top \mathbf{\Lambda}_{\mathbf{X}_\text{tr}} \mathbf{J}^*_{\mathbf{X}_\text{tr}} + 1/\sigma_0^2\mathbf{I}_{\text{dim}(\vtheta)}\right]^{-1}\\
&= \sigma_0^2 \left(\mathbf{I}_{\text{dim}(\vtheta)} - {\mathbf{J}^*_{\mathbf{X}_\text{tr}}}^\top \left[1/\sigma_0^2\mathbf{\Lambda}_{\mathbf{X}_\text{tr}}^{-1}+\mathbf{J}^*_{\mathbf{X}_\text{tr}} {\mathbf{J}^*_{\mathbf{X}_\text{tr}}}^\top\right]^{-1} {\mathbf{J}^*_{\mathbf{X}_\text{tr}}}\right).
\end{aligned}
\end{equation*}
Consequently, the covariance in \cref{eq:la-naive} becomes:
\begin{align*}
\scriptsize
&\mathbf{J}^*_{\vx} \mathbf{\Sigma}{\mathbf{J}^*_{\vx'}}^\top\\
=&\sigma_0^2 \mathbf{J}^*_{\vx} \left(\mathbf{I}_{\text{dim}(\vtheta)} - {\mathbf{J}^*_{\mathbf{X}_\text{tr}}}^\top \left[1/\sigma_0^2\mathbf{\Lambda}_{\mathbf{X}_\text{tr}}^{-1}+\mathbf{J}^*_{\mathbf{X}_\text{tr}} {\mathbf{J}^*_{\mathbf{X}_\text{tr}}}^\top\right]^{-1} {\mathbf{J}^*_{\mathbf{X}_\text{tr}}}\right) {\mathbf{J}^*_{\vx'}}^\top \\
=& \sigma_0^2 \left(\mathbf{J}^*_{\vx}{\mathbf{J}^*_{\vx'}}^\top - \mathbf{J}^*_{\vx}{\mathbf{J}^*_{\mathbf{X}_\text{tr}}}^\top \left[1/\sigma_0^2\mathbf{\Lambda}_{\mathbf{X}_\text{tr}}^{-1}+\mathbf{J}^*_{\mathbf{X}_\text{tr}} {\mathbf{J}^*_{\mathbf{X}_\text{tr}}}^\top\right]^{-1} {\mathbf{J}^*_{\mathbf{X}_\text{tr}}}{\mathbf{J}^*_{\vx'}}^\top \right) \\
=& \sigma_0^2 \left(\kappa_\text{NTK}(\vx, \vx') - \kappa_\text{NTK}(\vx, {\mathbf{X}_\text{tr}}) \left[1/\sigma_0^2\mathbf{\Lambda}_{\mathbf{X}_\text{tr}}^{-1}+\kappa_\text{NTK}({\mathbf{X}_\text{tr}}, {\mathbf{X}_\text{tr}})\right]^{-1} \kappa_\text{NTK}({\mathbf{X}_\text{tr}}, \vx') \right)\tag*{($\kappa_\text{NTK}(\vx, \vx'):=\partial_{\vtheta}g(\vx, \vtheta_\text{MAP}) \partial_{\vtheta}g(\vx', \vtheta_\text{MAP})^\top={\mathbf{J}^*_{\vx}}{\mathbf{J}^*_{\vx'}}^\top$)} \\
\approx & \sigma_0^2 \left(\tilde{\Psi}(\vx)\tilde{\Psi}(\vx')^\top - \tilde{\Psi}(\vx)\tilde{\Psi}({\mathbf{X}_\text{tr}})^\top \left[1/\sigma_0^2\mathbf{\Lambda}_{\mathbf{X}_\text{tr}}^{-1}+\tilde{\Psi}({\mathbf{X}_\text{tr}})\tilde{\Psi}({\mathbf{X}_\text{tr}})^\top\right]^{-1} \tilde{\Psi}({\mathbf{X}_\text{tr}})\tilde{\Psi}(\vx')^\top \right)\tag*{(Mercer's theorem)} \\
= &   \tilde{\Psi}(\vx) \sigma_0^2 \left(\mathbf{I}_k - \tilde{\Psi}({\mathbf{X}_\text{tr}})^\top \left[1/\sigma_0^2\mathbf{\Lambda}_{\mathbf{X}_\text{tr}}^{-1}+\tilde{\Psi}({\mathbf{X}_\text{tr}})\tilde{\Psi}({\mathbf{X}_\text{tr}})^\top\right]^{-1} \tilde{\Psi}({\mathbf{X}_\text{tr}}) \right) \tilde{\Psi}(\vx')^\top\\
= & \tilde{\Psi}(\vx)  \left[\tilde{\Psi}({\mathbf{X}_\text{tr}})^\top \mathbf{\Lambda}_{\mathbf{X}_\text{tr}} \tilde{\Psi}({\mathbf{X}_\text{tr}}) + {1}/{\sigma_0^2}\mathbf{I}_k\right]^{-1} \tilde{\Psi}(\vx')^\top\tag*{(Woodbury matrix identity)} \\ 
= & \tilde{\Psi}(\vx)  \left[\sum_i \tilde{\Psi}(\vx_i)^\top \mathbf{\Lambda}_i \tilde{\Psi}(\vx_i) + {1}/{\sigma_0^2}\mathbf{I}_k\right]^{-1} \tilde{\Psi}(\vx')^\top \\ 
\end{align*}
where $\tilde{\Psi}(\vx) := [\sqrt{\hat{\mu}_1}\hat{\psi}_1(\vx), ..., \sqrt{\hat{\mu}_k}\hat{\psi}_k(\vx)] \in \mathbb{R}^{N_\text{out} \times k}$ with $\hat{\psi}_i: \mathcal{X}\rightarrow \mathbb{R}^{N_\text{out}}$ as the approximate multi-output eigenfunction corresponding to the approximate $i$-th largest eigenvalues $\hat{\mu}_i$ of $\kappa_\text{NTK}$. 
Note that $\tilde{\Psi}({\mathbf{X}_\text{tr}})$ is the concatenation of $\tilde{\Psi}(\vx_1), ..., \tilde{\Psi}(\vx_N)$ and hence is of size $NN_\text{out} \times k$.

Thus, we obtain \cref{eq:la-our}.
\end{proof}

\section{Experiment Settings}
\label{app:exp}

\textbf{Implementation of the Nystr\"{o}m method} To find the top-$k$ eigenpairs, we use the \texttt{scipy.linalg.eigh} API to eigendecompose the kernel matrix $\kappa(\mathbf{X}_\text{tr}, \mathbf{X}_\text{tr}) \in \mathbb{R}^{N\times N}$ with the argument \texttt{subset\_by\_index} as $[N-k, N-1]$. 

\textbf{Experiments on classic kernels}
We use 3-layer MLP to instantiate the neural eigenfunctions.
The hidden size of the MLP is set as $32$. 
We use a mixture of \texttt{Sin} and \texttt{Cos} activations, namely, one half of the hidden neurons use \texttt{Sin} activation and the other half use \texttt{Cos} activation.

\textbf{Experiments on MLP-GP kernels} The architecture of concern is a 3-layer MLP.
To set up the MLP-GP kernels, for every linear layer, the prior on the weights is set as $\mathcal{N}(0, 2/\omega)$ with $\omega$ as the layer width, and the prior on the bias is set as $\mathcal{N}(0, 1)$.
To compute the kernel matrix on the training data, we instantiate a finitely wide MLP whose width is $16$, and perform MC estimation by virtue of the strategies in \cref{sec:scaleup}.
The number of MC samples $S$ is set as $10000$.

\textbf{Experiments on CNN-GP kernels} To set up the CNN-GP kernels, for every convolutional/linear layer, the prior on the weights is set as $\mathcal{N}(0, 2/\text{fan\_in})$, and the prior on the bias is set as $\mathcal{N}(0, 0.01)$.
To compute the kernel matrix on the training data, we instantiate a finitely wide CNN whose width is $16$, and perform MC estimation by virtue of the strategies in \cref{sec:scaleup}.
The number of MC samples $S$ is set as $2000$. 
We instantiate $\Scale[0.9]{\hat{\psi}}$ as the CNNs with the same architecture as the concerned CNN-GP kernel but augmented with batch normalization~\cite{ioffe2015batch}, and set the the layers width as $32, 64, 128$. 
We use $6000$ randomly sampled MNIST training images (due to resource constraint) to perform the Nystr\"{o}m method.
The used polynomial kernel and RBF kernel take the follows forms:
\begin{equation*}
    \kappa(x,x')=(0.001x^\top x' + 1)^{10}
\end{equation*}
\begin{equation*}
    \kappa(x,x')=\exp(-\Vert x - x'\Vert^2 / 2 / 100),
\end{equation*}
where the hyper-parameters are found by grid search. 

\textbf{Experiments on NTKs} 
For the trained NN classifier, we fuse the BN layers into the convolutional layers to get a compact model that possesses $269,034$ parameters. 
To compute the kernel matrix on the training data, we perform MC estimation by virtue of the strategies in \cref{sec:scaleup} and the number of MC samples $S$ is set as $4000$. 
We instantiate $\Scale[0.9]{\hat{\psi}}$ as widened ResNet-20 with widening factor of $2.0$.

\textbf{Experiments on the linearised Laplace approximation with NeuralEF} We train the CIFAR-10 classifiers with ResNet architectures for totally $150$ epochs under MAP principle. 
The optimization settings are identical to the above ones.
In particular, the weight decay is $10^{-4}$, thus we can estimate the prior variance $\sigma_0^2 = \frac{1}{50000 \times 10^{-4}} = 0.2$ where $50000$ is the number of training data $N$. 
After classifier training, we fuse the BN layers into the convolutional layers to get a compact model. 
To compute the kernel matrix on the training data, we perform MC estimation by virtue of the strategies in \cref{sec:scaleup} and the number of MC samples $S$ is set as $4000$. 
We instantiate $\Scale[0.9]{\hat{\psi}}$ as widened ResNet-20 with widening factor of $2.0$. 
We learn the matrix-valued NTK kernel with the strategy provided in \cref{sec:method}. 
When testing, we use $256$ MC samples from the function-space posterior to empirically estimate the predictive distribution as the categorical likelihood is not conjugate of Gaussian. 
We scale the sampled noises by $20$ given the observation that they are pretty tiny orginally.

\begin{figure*}[t]
\centering
\begin{minipage}{\linewidth}
\centering
    \begin{subfigure}[b]{0.97\linewidth}
    \centering
    \includegraphics[width=\linewidth]{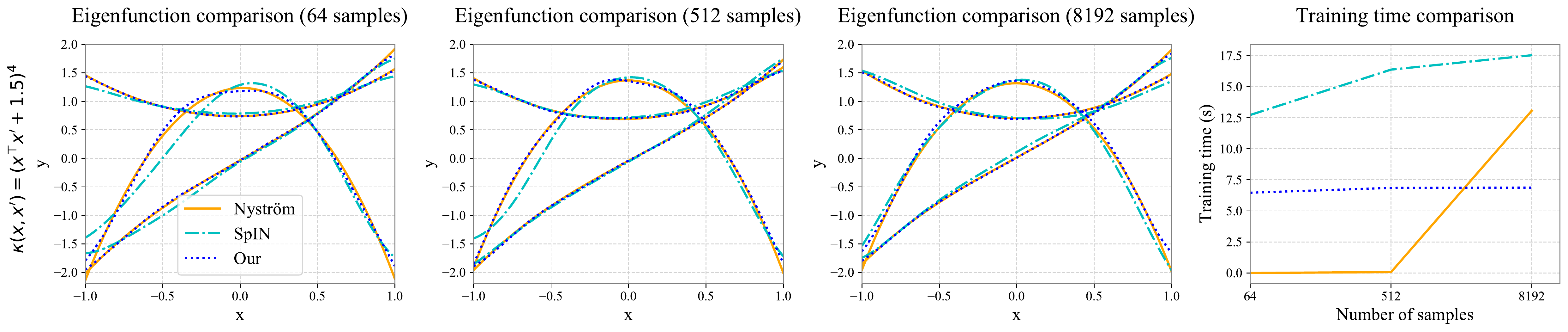}
        \vspace{-3.ex}
    \end{subfigure}
\end{minipage}
\begin{minipage}{\linewidth}
\vspace{-1ex}
\centering
    \begin{subfigure}[b]{0.97\linewidth}
    \centering
    \includegraphics[width=\linewidth]{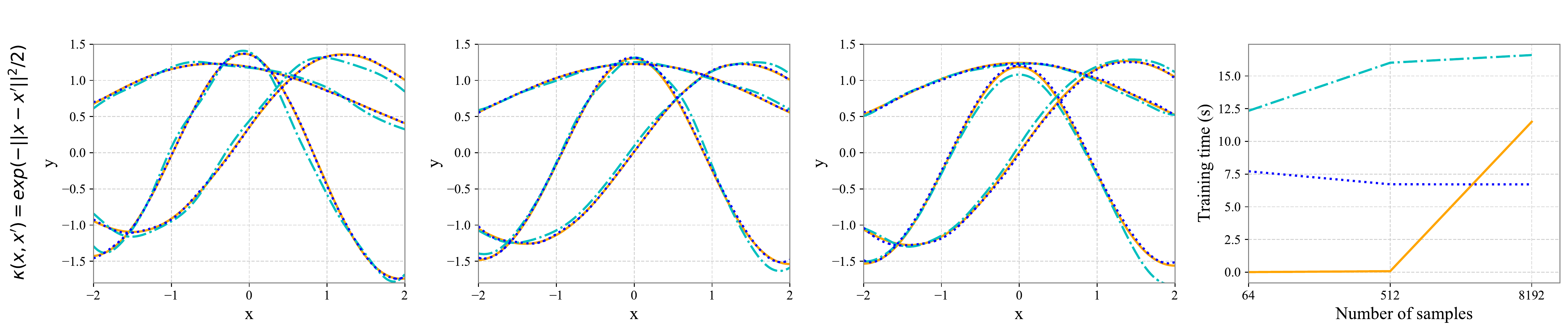}
        \vspace{-3.ex}
    \end{subfigure}
\end{minipage}
\caption{\footnotesize Estimate the eigenfunctions of the polynomial and RBF kernels with the Nystr\"{o}m method, SpIN, and NeuralEF (our). 
NeuralEF behaves as well as Nystr\"{o}m method but consumes \emph{nearly constant} training time w.r.t. sample size. 
} 
\label{fig:illu-sub}
\end{figure*}

\textbf{Experiments on the implicit kernel induced by SGD trajectory} To collect the SGD iterates, we first train the CIFAR-10 classifiers for $150$ epochs.
We set the initial learning rate as $0.1$, and scale the learning rate by $0.1$ at $80$-th and $120$-th epochs.
We use SGD optimizer with $0.9$ momentum to train the the classifiers, with the batch size set as $128$.
From $150$-th epoch to $155$-th epoch, we linearly increase the learning rate from $0.001$ to $0.05$, and then keep it constant until the classifiers have been trained for $200$ epochs. 
We totally collect $M=50$ weight samples for SWA and SWAG (one at per epoch). 
We cannot use more weight samples for constructing the Gaussian covariance in SWAG due to memory constraint. 
However, with NeuralEF introduced for kernel approximation, we can use many weight samples for defining $\kappa_\text{SGD}$ as eventually we save only the top eigenfunctions of the kernel instead of the kernel itself. 
In practice, we collect the function evaluations on the training set of around $M=4000$ weights from the SGD trajectory, based on which we estimate the training kernel matrix. 
We instantiate $\Scale[0.9]{\hat{\psi}}$ as widened ResNet-20 with widening factor of $2.0$. 
When testing, we use $256$ MC samples from the function-space posterior to empirically estimate the predictive distribution as the categorical likelihood is not conjugate of Gaussian.

\section{More Experiments}
\subsection{More Results on Polynomial and RBF kernels}
\label{app:classic}
We provide more results of various kernel approximation methods for the aforementioned polynomial and RBF kernels in \cref{fig:illu-sub}.

\subsection{Visualization of the Projections of MNIST Test Images}
\label{app:mnist-results}
We plot the top-$3$ dimensions of the projections belonging to the MNIST test images produced by NeuralEF in \cref{fig:mnist}. 
The NeuralEF model is trained to deconstruct the CNN-GP kernel mentioned in \cref{sec:exp-nngp}. 
We see the projections form class-conditional clusters, implying that NeuralEF can learn the discriminative structures in the CNN-GP kernel.

\begin{figure*}[t]
\centering
\begin{subfigure}[b]{0.99\linewidth}
\centering
\includegraphics[width=0.3\linewidth]{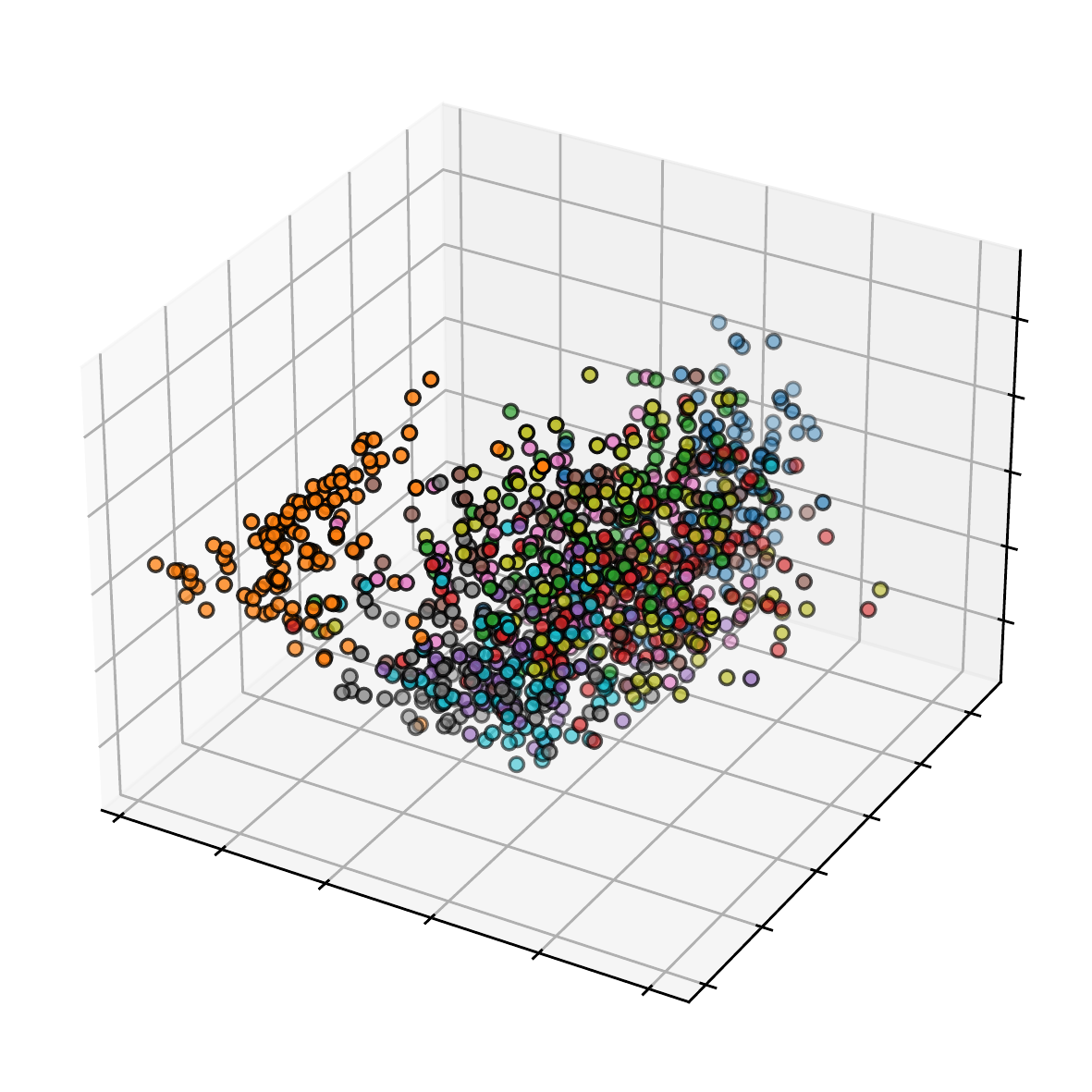}
    \vspace{-4ex}
\end{subfigure}
\caption{\footnotesize Project MNIST test images to the 3-D space by the approximate top-$3$ eigenfunctions of the CNN-GP kernel found by NeuralEF. 
Each color represents an category.}
\label{fig:mnist}
\vspace{-1.ex}
\end{figure*}

\subsection{Learn the Kernel Induced by SGD Trajectory}
\label{sec:sgd}

There has been a surge of interest in exploiting the SGD trajectory for Bayesian deep learning, attirbuted to the close connection between the stationary distribution of SGD iterates and the Bayesian posterior~\cite{mandt2017stochastic}. 
SWAG~\cite{maddox2019simple} is a typical method in this line: it collects a bunch of NN weights from the SGD trajectory then constructs approximate weight-space posteriors.
In fact, akin to the random feature approaches, the SGD iterates %
implicitly induce a kernel:
\begin{equation*}
\small
\kappa_\text{SGD}(\vx, \vx') = \frac{1}{M}\sum_{i=1}^M \left(g(\vx, \vtheta_i) - \bar{g}(\vx)\right)\left(g(\vx', \vtheta_i) - \bar{g}(\vx')\right)^\top,
\end{equation*}
where $g$ refers to an NN function, $\{\vtheta_i\}_{i=1}^M$ are the weights SGD traverses, and $\bar{g}(\cdot) = \frac{1}{M} g(\cdot, \vtheta_i)$ is the ensemble of SGD iterates. 
We can then define a Gaussian process $\mathcal{GP}(\bar{g}(\vx), \kappa_\text{SGD}(\vx, \vx'))$ to approximate the true function-space posterior, which leads to the \emph{posterior predictive}
\begin{equation*}
\small
p(\vx_\text{new}) = \int \mathcal{GP}(f|\bar{g}(\vx), \kappa_\text{SGD}(\vx, \vx'))p(\vx_\text{new} | f) df.
\end{equation*}

Yet, the evaluation of both $\bar{g}$ and $\kappa_\text{SGD}$ on a datum $\vx$ entails $M$ forward passes, resulting in the dilemma of deciding efficiency or exactness. 
Nonetheless, we can conjoin both by (\RN{1}) taking the Stochastic Weight Averaging (SWA)~\cite{izmailov2018averaging}, which predicts with the average weights, %
as a substitute for $\bar{g}$\footnote{\citet{maddox2019simple} found that SWA is basically on par with $\bar{g}$ in terms of performance.} and (\RN{2}) proceeding with the top eigenfunctions of $\kappa_\text{SGD}$ instead of the kernel itself.

To verify, we experiment on the full CIFAR-10 dataset with ResNets. 
Following the common assumption in GP classification, we assume that the kernel correlations among output dimensions are negligible, so $\vect{\kappa}^{\mathbf{X},\mathbf{X}}$ degrads to $N_\text{out}$ matrices of size $B \times B$. 
Viewing $\vect{\hat{\psi}_j}^{\mathbf{X}}$ as $N_\text{out}$ vectors of size $B$, we can then compute the losses \cref{eq:loss} dimension by dimension, sum them up, and invoke once backprop.
With this strategy, we find the top-$10$ eigenpairs at per output dimension, %
and then use them for kernel recovery and posteriori prediction.
We dub our model as SWA+NeuralEF. 
More experimental details are presented in \cref{app:exp}. 
The baselines of concern include the vanilla models trained by SGD, SWA, and SWAG. 
The results on test accuracy, negative log-likelihood (NLL), and expected calibration error (ECE)~\cite{guo2017calibration} are given in \cref{fig:cifar-sgd-trajectory}.
We further assess the models with ResNet-20 architecture on CIFAR-10 corruptions~\cite{hendrycks2019benchmarking}, a standard OOD testbed for CIFAR-10 models, with \cref{fig:cifar-corruption-sgd-trajectory} presenting the results.

We can see that SWA+NeuralEF is on par with or superior over SWAG across evaluation metrics, and outperforms SWA, especially in the aspect of ECE.
We emphasize that another merit of SWA+NeuralEF is that, unlike SWAG, the storage cost of SWA+NeuralEF is \emph{agnostic} to $M$.

\begin{figure*}[t]
\centering
\begin{subfigure}[b]{0.325\linewidth}
\centering
\includegraphics[width=\linewidth]{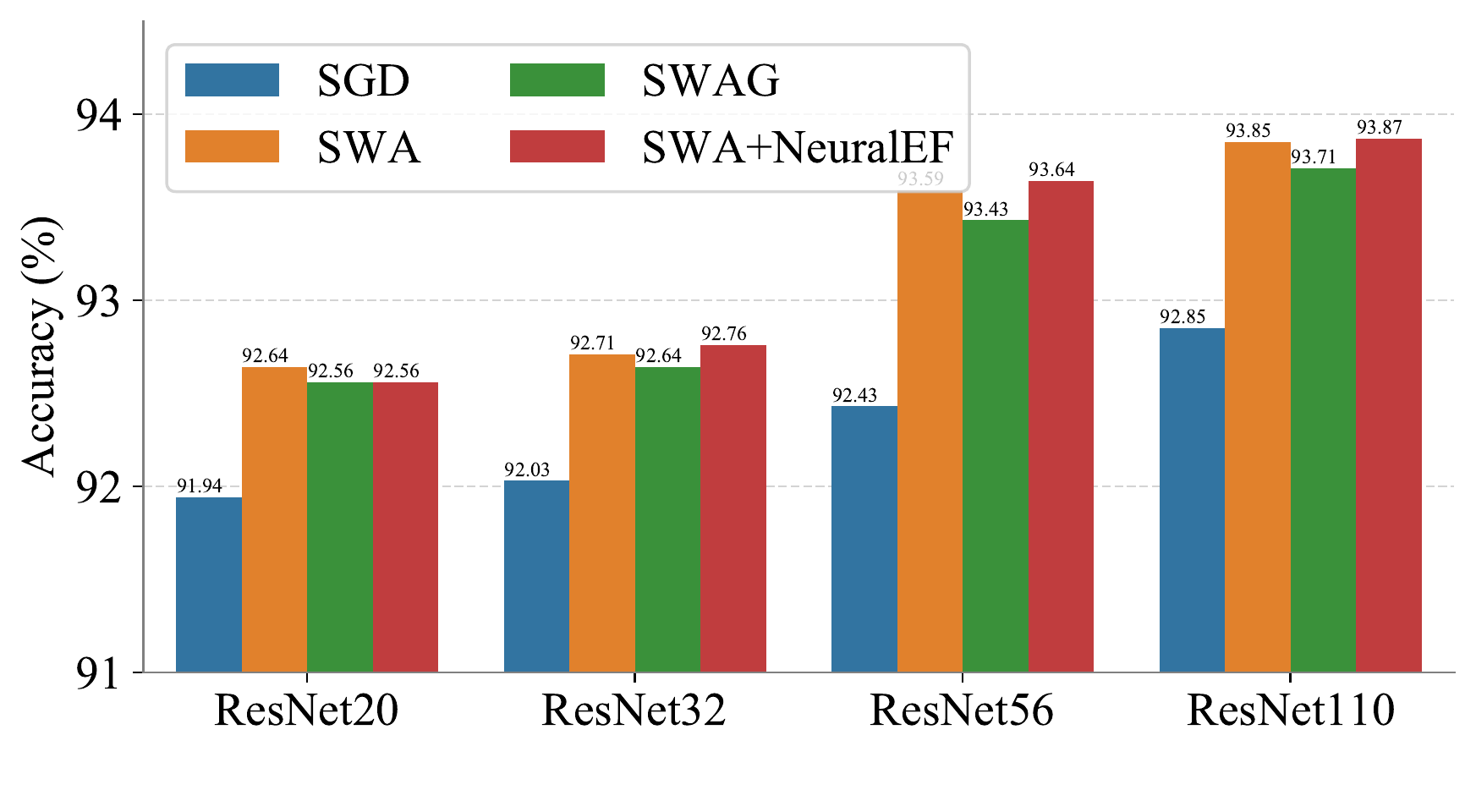}
    \vspace{-3.ex}
\end{subfigure}
\begin{subfigure}[b]{0.325\linewidth}
\centering
\includegraphics[width=\linewidth]{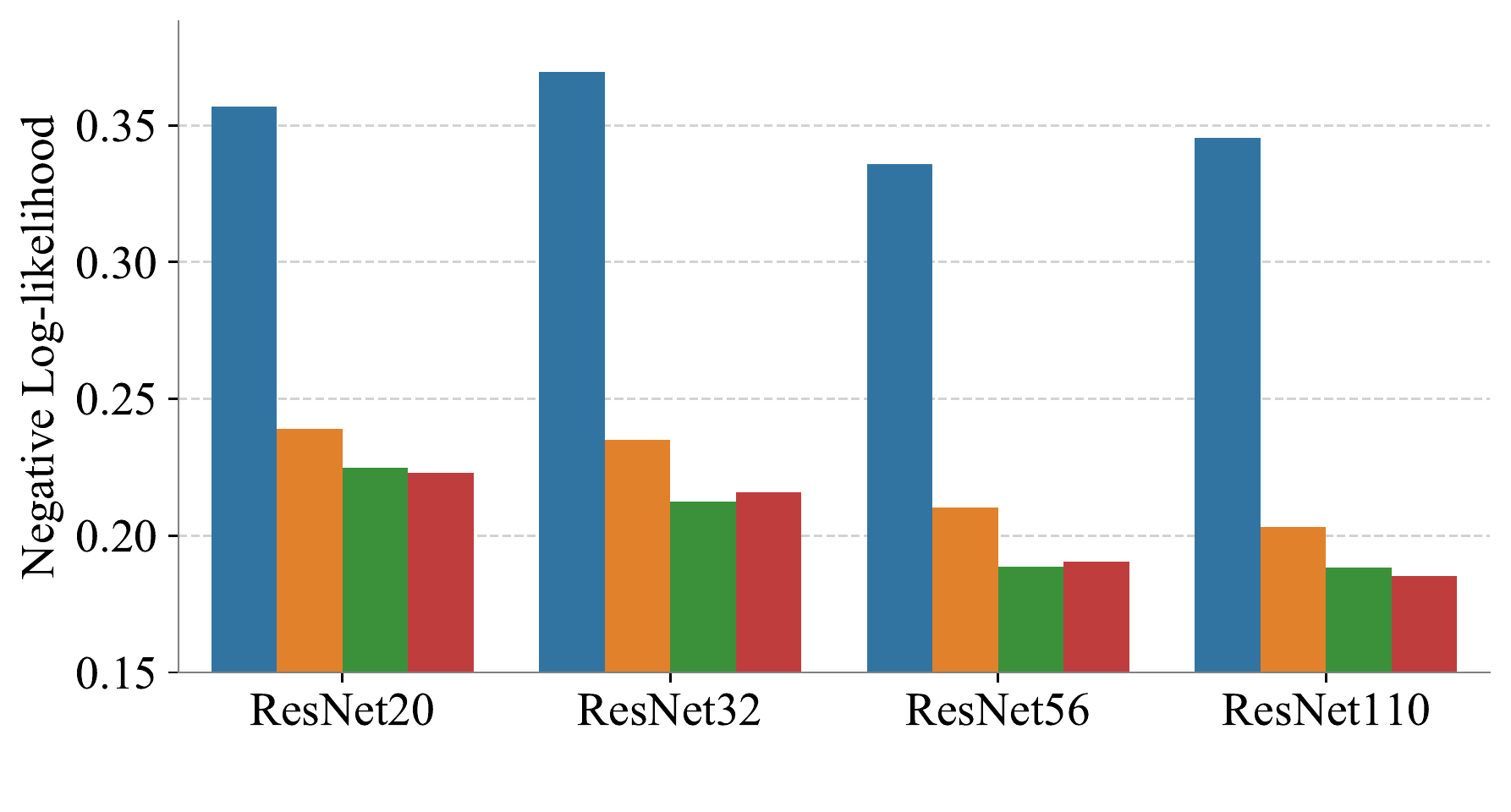}
    \vspace{-3.ex}
\end{subfigure}
\begin{subfigure}[b]{0.325\linewidth}
\centering
\includegraphics[width=\linewidth]{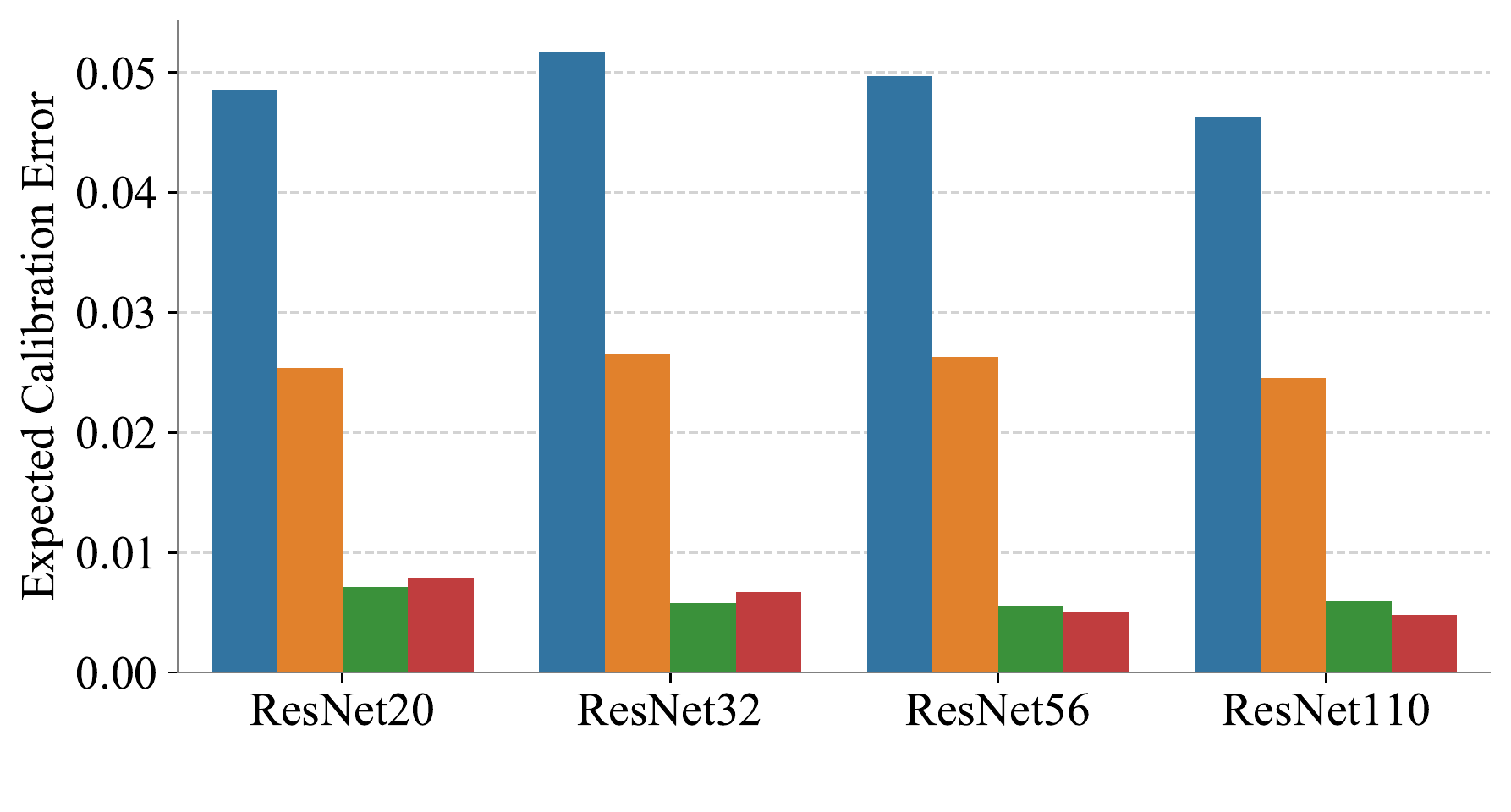}
    \vspace{-3.ex}
\end{subfigure}
\vspace{-2ex}
\caption{\footnotesize Test accuracy $\uparrow$, NLL $\downarrow$, and ECE $\downarrow$ comparisons among models on CIFAR-10. }
\label{fig:cifar-sgd-trajectory}
\end{figure*}

\begin{figure*}[t]
\centering
\begin{subfigure}[b]{0.49\linewidth}
\centering
\includegraphics[width=\linewidth]{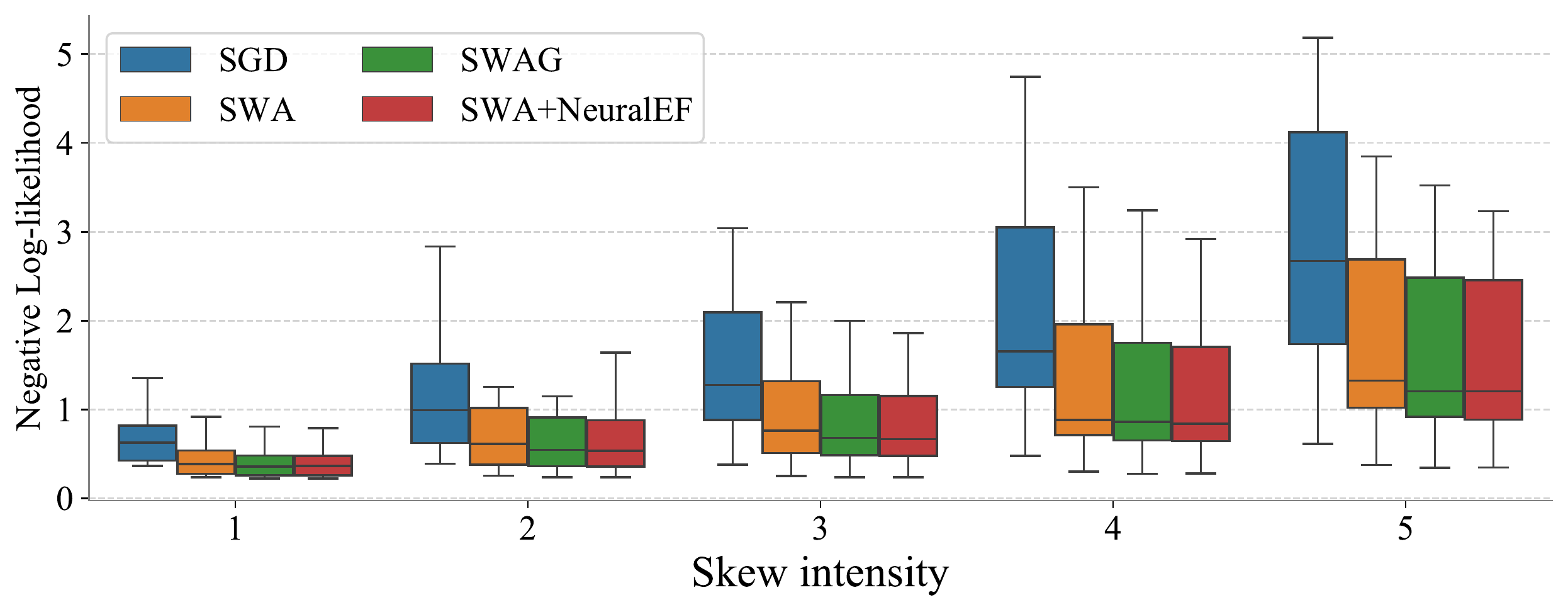}
    \vspace{-3.ex}
\end{subfigure}
\begin{subfigure}[b]{0.49\linewidth}
\centering
\includegraphics[width=\linewidth]{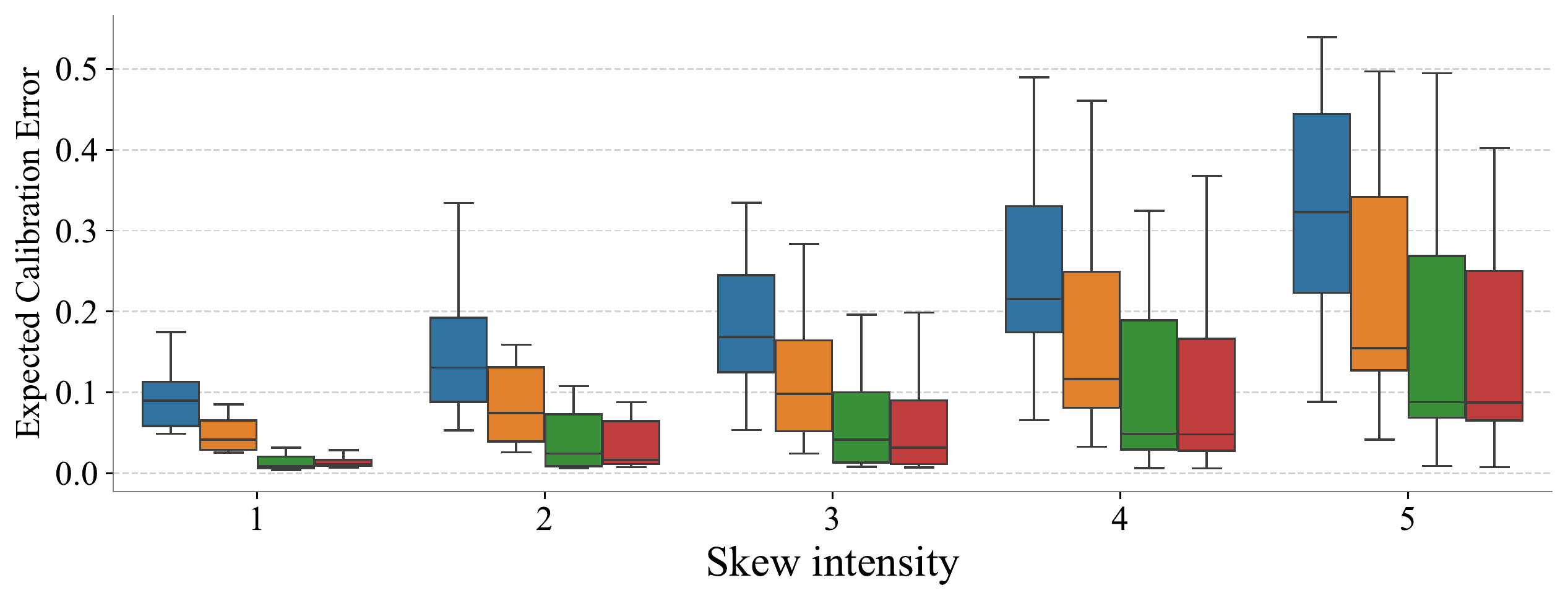}
    \vspace{-3.ex}
\end{subfigure}
\vspace{-1ex}
\caption{\footnotesize NLL $\downarrow$ and ECE $\downarrow$ on CIFAR-10 corruptions for models trained with ResNet-20 architecture. The results across 19 types of skew are summarized in the boxes.}
\label{fig:cifar-corruption-sgd-trajectory}
\vspace{-1ex}
\end{figure*}

\end{document}

%% file: math_commands.tex
\usepackage{amsmath,amsfonts,bm}

\def\eqref#1{equation~\ref{#1}}

\def\1{\bm{1}}

\def\vtheta{{\bm{\theta}}}

\def\vg{{\bm{g}}}

\def\vv{{\bm{v}}}
\def\vw{{\bm{w}}}
\def\vx{{\bm{x}}}
\def\vy{{\bm{y}}}

\DeclareMathAlphabet{\mathsfit}{\encodingdefault}{\sfdefault}{m}{sl}
\SetMathAlphabet{\mathsfit}{bold}{\encodingdefault}{\sfdefault}{bx}{n}